\documentclass{article}

\usepackage{microtype}
\usepackage{graphicx}
\usepackage{comment}
\usepackage{subfigure}
\usepackage{booktabs} 
\usepackage{thm-restate}

\usepackage{xcolor}
\usepackage{hyperref}

\newcommand{\expnum}[2]{#1\mathrm{e}^{-#2}}

\usepackage[preprint]{icml2026}

\usepackage{amsmath}
\usepackage{amssymb}
\usepackage{mathtools}
\usepackage{amsthm}
\usepackage{enumitem}
\usepackage[most]{tcolorbox}

\usepackage[capitalize,noabbrev]{cleveref}

\theoremstyle{plain}
\newtheorem{theorem}{Theorem}[section]
\newtheorem{proposition}[theorem]{Proposition}
\newtheorem{lemma}[theorem]{Lemma}

\newcommand{\R}{\mathbb{R}}

\newcommand{\EE}{\mathbb{E}}

\DeclareMathOperator{\tr}{tr}

\newcommand{\Var}{\mathrm{Var}}
\newcommand{\Cov}{\mathrm{Cov}}
\newcommand{\Denij}{\bigl(1+\tfrac{v_j-v_i}{m\,d_{ij}^{2}}\bigr)^{2}}

\usepackage[textsize=tiny]{todonotes}

\icmltitlerunning{Directional Neural Collapse Explains Few-Shot Transfer in Self-Supervised Learning}

\begin{document}

\twocolumn[
\icmltitle{Directional Neural Collapse Explains Few-Shot Transfer \\ in Self-Supervised Learning}



\icmlsetsymbol{equal}{*}

\begin{icmlauthorlist}
\icmlauthor{Achleshwar Luthra}{equal,yyy}
\icmlauthor{Yash Salunkhe}{equal,yyy}
\icmlauthor{Tomer Galanti}{equal,yyy}
\end{icmlauthorlist}

\icmlaffiliation{yyy}{Department of Computer Science \& Engineering, Texas A\&M University}

\icmlcorrespondingauthor{Tomer Galanti}{galanti@tamu.edu}

\icmlkeywords{Machine Learning, ICML}

\vskip 0.3in
]



\printAffiliationsAndNotice{\icmlEqualContribution} 

\begin{abstract}
Frozen self-supervised representations often transfer well with only a few labels across many semantic tasks. We argue that a single geometric quantity, \emph{directional} CDNV (decision-axis variance), sits at the core of two favorable behaviors: strong few-shot transfer within a task, and low interference across many tasks. We show that both emerge when variability \emph{along} class-separating directions is small. First, we prove sharp non-asymptotic multiclass generalization bounds for downstream classification whose leading term is the directional CDNV. The bounds include finite-shot corrections that cleanly separate intrinsic decision-axis variability from centroid-estimation error. Second, we link decision-axis collapse to multitask geometry: for independent balanced labelings, small directional CDNV across tasks forces the corresponding decision axes to be nearly orthogonal, helping a single representation support many tasks with minimal interference. Empirically, across SSL objectives, directional CDNV collapses during pretraining even when classical CDNV remains large, and our bounds closely track few-shot error at practical shot sizes. Additionally, on synthetic multitask data, we verify that SSL learns representations whose induced decision axes are nearly orthogonal. The code and project page of the paper are available at [\href{https://dlfundamentals.github.io/directional-neural-collapse/}{project page}].
\end{abstract}

\section{Introduction}

\begin{figure}[t]
    \centering
 \includegraphics[width=0.48\linewidth]{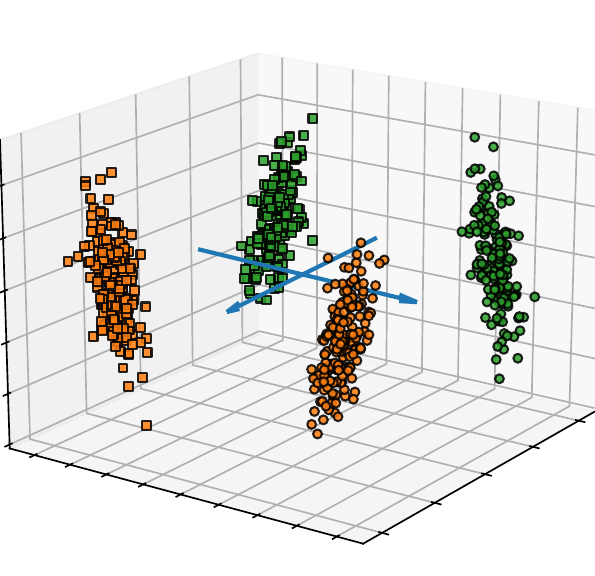} 
\caption{\textbf{Directional collapse and multitask orthogonalization in SSL.} Self-supervised pretraining suppresses within-class variance \emph{along} class-separating directions (small directional CDNV) while leaving substantial variance in orthogonal, task-irrelevant subspaces. When directional CDNV is small for multiple independent labelings, the induced decision axes become nearly orthogonal, enabling a single representation to support many tasks with low interference.}
    \label{fig:multitask}
\end{figure}

Self-supervised learning (SSL)~\cite{balestriero2023cookbookselfsupervisedlearning} has become a standard way to pretrain visual and multimodal representations without labels
.
A striking empirical fact is that frozen SSL features often enable strong few-shot transfer, with only a handful of labeled examples per class, across multiple downstream tasks simultaneously.
Yet we still lack a clean geometric explanation for when and why this behavior occurs.
This paper asks:
\begin{tcolorbox}[colback=blue!5!white, colframe=blue!10!black, arc=1pt]
\centering
\textbf{\em Given a fixed SSL representation, what geometric properties enable effective few-shot adaptation across multiple tasks simultaneously?}
\end{tcolorbox}

A useful point of comparison is supervised training, where representation geometry is better understood. In many supervised classifiers, late-layer features exhibit \emph{neural collapse} (NC)~\cite{doi:10.1073/pnas.2015509117,han2022neural}: within each class, embeddings concentrate near a single mean; across classes, means spread in an approximately simplex-like configuration; and the linear classifier aligns with these directions. This geometry directly \emph{relates} to few-shot success: when within-class dispersion is small relative to between-class separation, simple downstream rules such as nearest-class-centroid (NCC) achieve low error from limited data. Empirically, few-shot transfer and downstream transferability correlate with reduced within-class variability relative to between-class separation, as observed in analyses of few-shot representations and neural-collapse-inspired transferability metrics on the target data~\cite{pmlr-v119-goldblum20a,galanti2022on,10377311,li2024understanding}. 

Empirically, SSL representations do exhibit clustering and separability with respect to semantic labels~\cite{pmlr-v119-chen20j,Caron_2021_ICCV,oquab2024dinov2learningrobustvisual,shaul2023reverse,weng2025clusteringpropertiesselfsupervisedlearning}, even though labels are never used during pretraining. For example, \cite{shaul2023reverse} shows that SSL training induces semantic label-based clustering and increasingly strong nearest-class-centroid (NCC) separability, this effect strengthens during training and in deeper layers and appears simultaneously for multiple labelings in parallel. 
These findings suggest an NC-like phenomenon in SSL, but also highlight a key difference: the geometry that matters for decisions can improve without requiring global within-class collapse in all directions.

SSL, however, is not trained to induce global within-class collapse. Because labels are absent during pretraining, there is no direct pressure to reduce \emph{total} within-class variance. Empirically, SSL embeddings are often strongly \emph{anisotropic}: substantial variance persists in directions that do not affect class decisions (e.g., nuisance or augmentation-induced directions), while the directions that \emph{do} separate classes can be well organized. As a result, global clustering proxies such as the class-distance-normalized-variance (CDNV)~\cite{galanti2022on} can be too coarse and yield pessimistic or misleading predictions of few-shot transfer.

Recent work begins to address this mismatch by measuring variability only \emph{along class-separating directions}~\cite{luthra2025selfsupervisedcontrastivelearningapproximately}. They introduce a \emph{directional} analogue of CDNV that keeps the component of within-class variance that perturbs the decision margin while ignoring variance in orthogonal subspaces. This is the right geometric target for anisotropic SSL representations, and it suggests that few-shot error can be small even when total variance remains large. However, existing bounds based on this quantity are often loose (and sometimes vacuous) at practical shot sizes, and the empirical evidence is limited to a narrow set of settings (e.g., ResNet-50 with SimCLR).

\subsection{Contributions}
We study few-shot transfer from a fixed self-supervised representation through variability \emph{along} class-separating directions. Our main contributions are:
\begin{itemize}[leftmargin=1.25em]
\item \textbf{Sharp few-shot guarantees driven by decision-axis variability.} We prove non-asymptotic multiclass error bounds for nearest-class-centroid (NCC) and linear probing (LP) classification whose leading term is governed by \emph{directional} CDNV (decision-axis variance), rather than classical CDNV which aggregates variance over all directions (as in~\cite{galanti2022on,galanti2022improved,galanti2023generalizationboundsfewshottransfer}). Our analysis makes the shot dependence explicit through finite-sample centroid-estimation terms and a fourth-moment correction for heavy tails, yielding realistic error estimates across a wide range of shot sizes
(Theorems~\ref{thm:eq-wts} and~\ref{thm:ncc-full-optimized}).

\item \textbf{Decision-axis collapse yields accurate, non-vacuous few-shot certificates.} Across diverse SSL encoders, and objectives, we find that decision-axis variability collapses strongly during SSL training even when total within-class variance remains large, revealing a pervasive \emph{anisotropic} geometry where most variability lies in directions that do not affect discrimination.
In this regime, our bound closely tracks observed few-shot error and is substantially more informative than clustering-based proxies and prior directional bounds that can be loose at practical shot sizes (e.g.,~\cite{luthra2025selfsupervisedcontrastivelearningapproximately}).
On the theory side, we show that the leading coefficient in front of $\tilde V_{ij}$ is \emph{optimal}: in the known-centroid limit, pairwise NCC error is a one-dimensional tail event along the separating axis, and Cantelli's inequality (with its tight two-point extremizer) implies that no distribution-free second-moment bound can improve the factor $4$ (App.~\ref{app:opt_const_4}).

\item \textbf{Multitasks geometry: small decision-axis variance forces near-orthogonality of decision directions.} We prove a structural consequence of simultaneously small decision-axis variability across tasks: for two independent balanced binary labelings, small directional CDNV forces the corresponding decision axes to be nearly orthogonal (Prop.~\ref{prop:near_orth_from_dircdnv_multiclass}).
We complement this with a simple factor model showing how a single representation can support many tasks with small decision-axis variance even when classical CDNV is large, since most within-class energy can concentrate in directions orthogonal to all task-relevant axes (Sec.~\ref{sec:factor_model_cdnv_dircdnv}). Empirically, on controlled synthetic data with independent visual factors (e.g., shape, size, color, pattern), we find that SSL encoders map distinct factors to approximately orthogonal directions, consistent with this multitask prediction.
\end{itemize}

\section{Related Work}\label{sec:related}

{\bf Clustering properties in SSL.\enspace} Although SSL is trained without labels, learned representations often align with semantic categories and support simple downstream rules. Empirically, this occurs even in objectives that do not explicitly cluster representations: contrastive and non-contrastive pretraining can yield features with strong $k$-NN and linear-probe performance and increasingly clear semantic structure in deeper layers and later training~\cite{pmlr-v119-chen20j,Caron_2021_ICCV,shaul2023reverse,weng2025clusteringpropertiesselfsupervisedlearning,oquab2024dinov2learningrobustvisual}. These findings motivate geometric explanations of transfer that characterize how class-conditional distributions are arranged in embedding space, rather than focusing only on instance discrimination.

At the same time, SSL representations are often strongly \emph{anisotropic}: substantial within-class variability can persist in nuisance or augmentation-induced directions that have little effect on discrimination. This behavior has been documented empirically and theoretically in analyses of contrastive and non-contrastive objectives and alignment–uniformity trade-offs~\cite{wang2020understanding,wang2021understanding,chen2021intriguing,NEURIPS2021_27debb43,arora2019theoreticalanalysiscontrastiveunsupervised}, and is also implicit in the design of modern SSL methods that explicitly discourage global collapse by preserving variance and decorrelating features~\cite{zbontar2021barlow,bardes2022vicreg}. As a result, global clustering proxies that aggregate variance across all directions can be too coarse to reliably predict few-shot performance.

{\bf Theoretical analyses of SSL.\enspace}
A growing literature studies why self-supervised learning works. Early analyses of contrastive learning related InfoNCE-like objectives to mutual information~\cite{NEURIPS2019_ddf35421}, while later work noted that strict MI constraints can be overly restrictive~\cite{pmlr-v108-mcallester20a,Tschannen2020On}. Another influential line characterizes representations via \emph{alignment} and \emph{uniformity}~\cite{wang2020understanding,wang2021understanding,chen2021intriguing}. These descriptors are informative, but they do not directly specify how \emph{different} semantic classes are arranged.

This gap motivated analyses of when contrastive objectives recover clusters or latent structure~\cite{arora2019theoreticalanalysiscontrastiveunsupervised,tosh2021contrastive,zimmermann2021contrastive,pmlr-v151-ash22a,NEURIPS2021_2dace78f,NEURIPS2021_27debb43,10.5555/3600270.3602219,pmlr-v162-shen22d,wang2022chaos,pmlr-v162-awasthi22b,pmlr-v162-bao22e}, often under strong assumptions such as conditional independence of augmentations given a latent cluster~\cite{arora2019theoreticalanalysiscontrastiveunsupervised,pmlr-v162-saunshi22a,tosh2021contrastive,pmlr-v162-awasthi22b}. Recent work relaxes aspects of this picture via alternative function-class biases~\cite{haochen2023a} or capacity limitations that discourage splitting semantic clusters under InfoNCE~\cite{pmlr-v195-parulekar23a}. A complementary perspective connects SSL objectives to supervised criteria in simplified settings~\cite{balestriero2024the,luthra2025selfsupervisedcontrastivelearningapproximately,luthra2025alignmentsupervisedselfsupervisedcontrastive}.

Our work is closest to analyses that certify \emph{downstream} performance from frozen representations. Most closely, \citet{luthra2025selfsupervisedcontrastivelearningapproximately} introduce \emph{directional} CDNV, which measures within-class variability only along class-separating axes and yields an anisotropy-aware certificate for NCC. We build on this directional viewpoint in three ways: (i) we give sharper, non-asymptotic multiclass guarantees for both NCC and LP whose leading term is governed by directional CDNV, including explicit finite-shot centroid-estimation corrections and a separate fourth-moment term isolating tail effects, with an optimal leading constant under second-moment information; (ii) we provide broad empirical evidence that decision-axis variability collapses across diverse SSL paradigms (contrastive, non-contrastive, masked prediction, distillation, and multimodal pretraining), suggesting that suppressing variance along discriminative directions is an implicit, method-agnostic outcome even when total within-class variance remains large; and (iii) we connect this anisotropic geometry to multitask structure by showing that if directional CDNV is simultaneously small for independent balanced labelings, then the associated decision directions must be nearly orthogonal, formalizing how a single representation can support many tasks while classical CDNV stays large.

Beyond these strands, theory has examined learning dynamics in simplified models~\cite{NEURIPS2022_7b5c9cc0,10.5555/3648699.3649029,pmlr-v139-wen21c,tian2023understanding}, the role of augmentations~\cite{NEURIPS2020_4c2e5eaa,feigin2025theoreticalcharacterizationoptimaldata}, projection heads~\cite{gupta2022understandingimprovingroleprojection,gui2023unravelingprojectionheadscontrastive,xue2024investigating,ouyang2025projection}, sample complexity~\cite{alon2024optimal,DBLP:conf/icml/YuanWQDZZY22}, and links between contrastive and non-contrastive methods~\cite{wei2021theoretical,balestriero2022contrastive,NEURIPS2021_02e656ad,garrido2023on,shwartzziv2023an}.

{\bf Neural collapse and geometric predictors of transfer.\enspace}
Neural collapse (NC) \cite{doi:10.1073/pnas.2015509117,han2022neural} describes a common end-of-training geometry in supervised classifiers: (i) \emph{within-class collapse}, where embeddings concentrate near class means; (ii) \emph{class separation}, where means form a nearly simplex equiangular tight frame (ETF); and (iii) \emph{alignment} between classifier weights and feature geometry.
NC has motivated geometric predictors of generalization and transfer, and neural-collapse-inspired measures on the \emph{target} dataset have been shown to correlate with transferability across pretrained models~\cite{10377311}.
In few-shot regimes, representation analyses likewise emphasize the role of within-class compactness relative to between-class separation for centroid-based decision rules~\cite{pmlr-v119-goldblum20a,galanti2022on}. Our goal is to extend this NC perspective to SSL in a way that respects anisotropy: rather than requiring global collapse, we characterize an SSL analogue that demands ``collapse'' only along decision-relevant directions, and we derive practical few-shot certificates that match this geometry.

\section{Problem Setup}\label{sec:setting}

The primary goal of self-supervised learning (SSL) is to learn representations that are easily adaptable to downstream tasks. We formalize the conditions under which an SSL-trained representation $f:\mathcal{X}\to\mathbb{R}^d$ supports accurate few-shot classification. We say that $f$ is \emph{good} if, after seeing only a few labeled examples per class, it yields low test error on the downstream task. We consider training an embedding function $f\in\mathcal{F}\subset\{\,f\mid f:\mathcal{X}\to\mathbb{R}^d\,\}$ via self-supervised learning, using an unlabeled dataset $X=\bigcup_{c=1}^{C}\{x_{c,i}\}_{i=1}^{n}\subset\mathcal{X}$ (e.g., images). Our goal is to learn a {\em ``meaningful''} function $f$ that maps samples to $d$-dimensional embedding vectors.

For instance, we can call a representation meaningful if the hidden class labels are recoverable from the samples.
Let $f:\mathcal{X}\to\mathbb{R}^d$ be a fixed representation learned by SSL. 
Fix a subset of classes $\mathcal{C}\subseteq [C]$ with $|\mathcal{C}|=C'$. 
For evaluation, draw $c\sim\mathrm{Unif}(\mathcal{C})$, then draw $x\sim D_c$, and set the label $y:=c$. Let $D_{\mathcal{C}} \;:=\; \tfrac{1}{C'}\sum_{c\in\mathcal{C}} D_c$ denote the balanced mixture over the selected classes. 
Given a classifier $h:\mathbb{R}^d\to \mathcal{C}$, its error on this mixture is
\[
\mathrm{err}_{D_{\mathcal{C}}}(h)
\;:=\; \Pr_{c\sim\mathrm{Unif}(\mathcal{C}),\,x\sim D_c} \big(h(f(x))\neq c\big).
\]

For each $c\in\mathcal C$, draw an i.i.d.\ support set $\widehat S_c := \{(x_{c,s},c)\}_{s=1}^{m}\sim D_c^m$, independent across classes and independent of the test draw $(c,x)$. Let $\widehat S := \bigcup_{c\in\mathcal C}\widehat S_c$. A downstream classifier $h_{f,\widehat S}$ is trained on $\widehat S$ and uses the features $f(x)$ at test time.\footnote{Any internal randomness of the training procedure is included in the expectation below.}
The expected $m$-shot error of $f$ is $\mathrm{err}_{m,\mathcal{C}}(f)
\;:=\; \EE_{\widehat S_1,\dots,\widehat S_{C'}}[
\mathrm{err}_{D_{\mathcal{C}}}\big(h_{f,\widehat S}\big)]$.

{\bf Downstream classifiers.\enspace} We study two simple downstream classifiers on top of the frozen encoder $f$: a linear probe (LP) and the nearest class centroid (NCC) rule. Given a support set $\widehat S=\{(x_{c,s},c): c\in\mathcal C,\ s\in[m]\}$ and a test feature $z=f(x)\in\mathbb R^{d}$, define
\begin{small}
\[
\widehat y^{\triangle}(z)\in\mathcal C
\quad\text{and}\quad
\mathrm{err}^{\triangle}_{m,\mathcal C}(f)
~:=~ \frac{1}{|\mathcal C|}\sum_{i\in\mathcal C}\Pr \big(\widehat y^{\triangle}(z_i)\neq i\big),
\]
\end{small}
where $\triangle\in\{\mathrm{LP},\mathrm{NCC}\}$ and $z_i:=f(x_i)$ with $x_i\sim D_i$ independent of $\widehat S$.

The LP is the minimizer of the population $0$–$1$ error over all linear rules $g(z)=\arg\max_{c\in\mathcal C}(w_c^\top z+b_c)$ on top of $f$, estimated from the support set $\widehat S$.
For NCC, let $z_{c,s}:=f(x_{c,s})$, $\mu_c:=\EE_{x\sim D_c}[f(x)]$, and $\widehat\mu_c:=\frac{1}{m}\sum_{s=1}^{m} z_{c,s}$ with $\delta_c:=\widehat\mu_c-\mu_c$.
Given a test feature $z=f(x)\in\mathbb R^{d}$, NCC predicts $\widehat y^{\mathrm{NCC}}(z):=\arg\min_{c\in\mathcal C}\|z-\widehat\mu_c\|_2^2
=\arg\max_{c\in\mathcal C}\bigl\{2\,\widehat\mu_c^\top z-\|\widehat\mu_c\|_2^2\bigr\}$, with an arbitrary but fixed tie-breaking rule.

Define the class-to-class error probability $p^{\triangle}_{i\to j}:=\Pr \big(\widehat y^{\triangle}(z_i)=j\big)$ for $i\neq j$, where the probability is over the draw of $\widehat S$ and the fresh test point $x_i$.
A classwise union bound yields $\mathrm{err}^{\triangle}_{m,\mathcal C}(f)
~\le~ \frac{1}{|\mathcal C|}\sum_{i\in\mathcal C}\ \sum_{\substack{j\in\mathcal C\\ j\neq i}} p^{\triangle}_{i\to j}$. Since NCC is a particular linear rule, it follows that $\mathrm{err}^{\mathrm{LP}}_{m,\mathcal C}(f)\le \mathrm{err}^{\mathrm{NCC}}_{m,\mathcal C}(f)$. To control $p^{\mathrm{NCC}}_{i\to j}$, introduce the \emph{pairwise margin} $\Delta_{i\to j}
:=\|z_i-\widehat\mu_j\|_2^2-\|z_i-\widehat\mu_i\|_2^2$ and $p^{\mathrm{NCC}}_{i\to j}=\Pr(\Delta_{i\to j}\le 0)$. Let $d_{ij}:=\|\mu_i-\mu_j\|_2$ and $v_c:=\EE \left[\|f(x)-\mu_c\|_2^2\,\middle|\,x\sim D_c\right]$ denote the class variances.
A direct expansion with $\EE[\delta_c]=0$ gives $\EE[\Delta_{i\to j}]
= d_{ij}^2+\tfrac{v_j-v_i}{m}$,
so larger between-class separation $d_{ij}$ and larger $m$ increase the expected margin, while variance asymmetry contributes the $\tfrac{v_j-v_i}{m}$ term. For simplicity, we assume throughout that $\EE[\Delta_{i\to j}]>0$.


\section{Theoretical Analysis}

\subsection{Background}

Above we described when a representation is favorable for few-shot adaptation; we now ask \emph{why} a representation yields small error. Classical analyses tie few-shot performance to the clustering geometry of class embeddings. 

{\bf Supervised case.\enspace} Prior results \cite{galanti2022on,galanti2022improved,galanti2023generalizationboundsfewshottransfer} bound these errors by the class-distance-normalized variance (CDNV). Writing $v_i:=\EE_{x\sim D_i}\|f(x)-\mu_i\|_2^2$ and
\( V_{ij} := V_f(D_i,D_j):=\frac{v_i+v_j}{\|\mu_i-\mu_j\|_2^2}\),
one obtains (see, e.g., Prop.~7 in \cite{galanti2023generalizationboundsfewshottransfer})
\begin{small}
\begin{equation}\label{eq:old_galanti}
\begin{aligned}
\mathrm{err}^{\mathrm{LP}}_{m,\mathcal{C}}(f)
~\le~
\mathrm{err}^{\mathrm{NCC}}_{m,\mathcal{C}}(f) ~\lesssim~
\frac{1+\tfrac{1}{m}}{C'} \sum_{\substack{i\in\mathcal{C}\\ j\in\mathcal{C},\, j\neq i}} V_{ij}.
\end{aligned}
\end{equation}
\end{small}
In supervised learning, representations tend to become tightly clustered, CDNV decreases, and \eqref{eq:old_galanti} tightens \cite{doi:10.1073/pnas.2015509117,galanti2022on,galanti2023comparative,zhou2022are}. In self-supervised learning, by contrast, there is no direct pressure to reduce \emph{total} within-class variance, so the average CDNV need not be small even when features are well organized along class-separating directions. 

Under supervised training, late-layer features often exhibit \emph{neural collapse} (NC): (i) within-class variance vanishes ($v_i \downarrow 0$); (ii) class means are centered and have equal norm; (iii) the means are approximately a simplex ETF, i.e., $\mu_i^\top\mu_j=r^2$ if $i=j$ and $\mu_i^\top\mu_j=-\tfrac{r^2}{C-1}$ otherwise (up to rotation/scale, with $r^2=\|\mu_i\|_2^2$ independent of $i$); and (iv) the final linear weights align with the means. Hence inter-class gaps $d_{ij}^2=\|\mu_i-\mu_j\|_2^2$ stay uniformly large while $v_i$ collapses, so $V_{ij}=\tfrac{v_i+v_j}{d_{ij}^2}$ shrinks and the CDNV average in~\eqref{eq:old_galanti} becomes small. This tightens the NCC bound (and thus LP), and explains why LP and NCC are often similar late in supervised training~\cite{doi:10.1073/pnas.2015509117,galanti2022on,galanti2023comparative,zhou2022are}.

{\bf Self-supervised case.\enspace}
In self-supervised pretraining, within-class variability can remain large in nuisance directions (e.g., augmentations) while variance \emph{along} class-separating directions is small. Classical CDNV averages over all directions and can therefore be pessimistic in such anisotropic regimes. Following \citet{luthra2025selfsupervisedcontrastivelearningapproximately}, we use \emph{directional CDNV}, which retains only the within-class variance that can flip a pairwise decision.

Let $z=f(x)$ for $x\sim D_c$, with class mean $\mu_c$ and covariance $\Sigma_c$. For classes $(i,j)$, define $d_{ij}=\|\mu_j-\mu_i\|_2$ and the decision axis $u_{ij}=(\mu_j-\mu_i)/d_{ij}$. The {\em directional CDNV} is $\tilde V_{ij} := \frac{u_{ij}^\top \Sigma_i\, u_{ij}}{d_{ij}^2}$, which measures class-$i$ variability \emph{along} the one-dimensional decision axis relative to the mean gap; variance orthogonal to $u_{ij}$ does not affect the pairwise margin. We also use averages $\tilde V_f=\textnormal{Avg}_{i\neq j}\tilde V_{ij}$, $V_f=\textnormal{Avg}_{i\neq j}V_{ij}$, and $V_f^s=\textnormal{Avg}_{i\neq j}\sqrt{V_{ij}}$.

Our anisotropic bounds replace the coarse dependence on $V_{ij}$ in \eqref{eq:old_galanti} with a leading term governed by $\tilde V_{ij}$, plus finite-shot corrections that depend on the global dispersion and decay with the number of shots $m$. Since linear probes optimize over linear rules, $\textnormal{err}^{\mathrm{LP}}_{m,\mathcal{C}}(f)\le \textnormal{err}^{\mathrm{NCC}}_{m,\mathcal{C}}(f)$. Concretely, \citet{luthra2025selfsupervisedcontrastivelearningapproximately} showed that for any $a\ge 5$,
\begin{small}
\begin{equation*}
\begin{aligned}
\textnormal{err}^{\mathrm{NCC}}_{m,\mathcal{C}}(f)
~\le~ (C'-1)\Bigg[
&\Big(\tfrac12-\tfrac{2}{a}-\tfrac{2^{3/2}}{am}\Big)^{-2}\tilde V_f \\
&+\tfrac{a}{4}\Big(\tfrac{2}{\sqrt m}(V_f^{s}+V_f)+\tfrac{1}{m}V_f\Big)
\Bigg].
\end{aligned}
\end{equation*}
\end{small}
Namely, NCC error decomposes into a decision-axis term and finite-shot leakage terms from estimating centroids (scaling with $V_f^s$ and $V_f$ at rates $m^{-1/2}$ and $m^{-1}$). Choosing $a=16$ yields a bound of $(C'-1)(
8\,\tilde V_f
+\frac{8}{\sqrt m}\,V_f^{s}
+[\frac{8}{\sqrt m}+\frac{4}{m}]\,V_f)$, making explicit that in anisotropic SSL regimes the dominant contribution is governed by $\tilde V_f$, while the remaining terms vanish as $m$ grows.

\subsection{Results}

While the directional CDNV of \cite{luthra2025selfsupervisedcontrastivelearningapproximately} correctly targets decision-axis variability, existing guarantees can be loose or vacuous at practical shot sizes (e.g., $m\in[1,500]$) and often entangle (i) the genuinely discriminative decision-axis term with (ii) finite-shot effects from estimating class centroids and (iii) tail behavior. Empirically, SSL embeddings are frequently \emph{anisotropic}: much of the within-class energy can lie in subspaces nearly orthogonal to the separating directions, motivating decorrelation/whitening objectives such as Barlow Twins and VICReg and alignment--uniformity analyses in contrastive representations \cite{zbontar2021barlow,bardes2022vicreg,wang2020understanding,wang2021understanding,chen2021intriguing}.

To obtain a usable certificate in anisotropic regimes, we analyze the \emph{pairwise margin} $\Delta_{i\to j}$ and derive bounds that (i) isolate the decision-axis second moment with an \emph{optimal} leading constant, (ii) make finite-shot effects from empirical centroids explicit, and (iii) quantify heavy tails via a fourth-moment correction. Concretely, we introduce
$\Theta_{ij}:=\tfrac{M_{4,i}+M_{4,j}}{d_{ij}^{4}}$, where $M_{4,i}:=\EE\|f(x)-\mu_i\|^{4}$, so that tail effects appear at order $m^{-3}$, while the remaining finite-shot leakage is controlled by $V_{ij}$ and $V_{ij}^{2}$ at orders $m^{-1}$ and $m^{-2}$. The result below combines the pairwise and multiclass guarantees, includes a mild variance-imbalance factor $\Denij$, and uses optimized coefficients.

\begin{restatable}{theorem}{optimized}\label{thm:ncc-full-optimized}
Let $C'\ge 2$ and $m\ge 10$ be integers. Fix a feature map $f:\mathcal X\to\mathbb R^{d}$ and class-conditional distributions $D_{1},\dots,D_{C'}$ over $\mathcal X$. Define $E^1_{ij}:=\frac{4}{m}(V_{ij}^2+\frac14 V_{ij})$, $E^2_{ij}:=\frac{V_{ij}}{m}$, $E^3_{ij}:=\frac{\Theta_{ij}+2(m-1)V_{ij}^{2}}{m^{3}}$. Then the average multiclass error of the NCC classifier satisfies
\begin{small}
\begin{equation*}
\begin{aligned}
\mathrm{err}^{\mathrm{NCC}}_{m,\mathcal{C}}(f)
~&\le~ \frac{1}{C'}\sum_{i=1}^{C'}\sum_{j\ne i}
\tfrac{4\,\tilde V_{ij}}{\Denij}\\
~&\qquad+~ \frac{1}{C'}\sum_{i=1}^{C'}\sum_{j\ne i}
\tfrac{\Bigl(\sqrt{E^1_{ij}}+\sqrt{E^2_{ij}}+\sqrt{E^3_{ij}}\Bigr)^{2}}{\Denij}.
\end{aligned}
\end{equation*}
\end{small}
\end{restatable}

The leading term $\frac{4\,\tilde V_{ij}}{\Denij}$ is the \emph{decision-axis} contribution, governed by the within-class variance of class $i$ along the separating direction $u_{ij}$. The remaining term aggregates all \emph{finite-shot} effects through
$\frac{\bigl(\sqrt{E^1_{ij}}+\sqrt{E^2_{ij}}+\sqrt{E^3_{ij}}\bigr)^2}{\Denij}$:
$E^2_{ij}$ captures linear centroid-estimation leakage ($\asymp V_{ij}/m$),
$E^1_{ij}$ captures quadratic leakage ($\asymp V_{ij}^2/m$),
and $E^3_{ij}$ captures tails and higher-order terms
($\asymp \Theta_{ij}/m^{3} + V_{ij}^{2}/m^{2}$, since $\frac{2(m-1)}{m^3}=\mathcal O(m^{-2})$).

In the large-shot limit $m\to\infty$, and under mild variance balance ($v_i\approx v_j$ so $\Denij\to 1$), the correction vanishes and the bound approaches the directional certificate $p^{\mathrm{NCC}}_{i\to j}\lesssim 4\,\tilde V_{ij}$ (and the multiclass bound follows by union bounding and averaging over $i$). The constant $4$ is optimal under second-moment information: in the known-centroid reduction the error is a one-sided tail event of the axis projection with variance $d_{ij}^2\tilde V_{ij}$, and Cantelli's inequality yields the minimax-tight factor $4$ (App.~\ref{app:opt_const_4}). Finally, $\Denij=\bigl(1+\tfrac{v_j-v_i}{m\,d_{ij}^2}\bigr)^2$ is a mild multiplicative perturbation whenever $\tfrac{|v_j-v_i|}{m\,d_{ij}^2}$ is not too large, and it disappears as $m$ grows.

{\bf Optimality of the leading constant.\enspace} The coefficient $4$ in the leading decision-axis term is not an artifact of our proof technique. In the idealized known-centroid limit ($m\to\infty$), the pairwise NCC error $p^{\mathrm{NCC}}_{i\to j}$
reduces to a one-dimensional tail event along the separating axis:
if $u_{ij}:=(\mu_j-\mu_i)/\|\mu_j-\mu_i\|_2$ and
$X:=(z_i-\mu_i)^\top u_{ij}$ for $z_i\sim D_i$, then
$p^{\mathrm{NCC}}_{i\to j}=\Pr \big(X\ge \|\mu_j-\mu_i\|_2/2\big)$ and
$\Var(X)=\|\mu_j-\mu_i\|_2^2\,\tilde V_{ij}$.
Among all mean-zero random variables with this variance, Cantelli's inequality yields the sharp
distribution-free bound
$p^{\mathrm{NCC}}_{i\to j}\le \frac{4\tilde V_{ij}}{1+4\tilde V_{ij}}\le 4\tilde V_{ij}$,
and the constant is minimax-tight via a two-point construction.
Consequently, any universally valid second-moment bound of the form
$p^{\mathrm{NCC}}_{i\to j}\le c\,\tilde V_{ij}+(\text{lower-order terms})$ must have $c\ge 4$
unless additional tail assumptions are imposed (e.g., subgaussianity).
See App.~\ref{app:opt_const_4} for details and the connection to the finite-shot factor $1/\Denij$.

\subsection{Near-orthogonality from small directional CDNV}
\label{sec:factor_model_cdnv_dircdnv}

A single SSL representation can support many downstream tasks even when the \emph{total} within-class spread is large.
The reason is that few-shot error (and our bounds) are controlled by variability \emph{along the decision axis}, not by variability in directions that are irrelevant to the margin.
Directional CDNV measures within-class variance projected onto the class-separating direction, whereas classical CDNV sums variance over all directions and can therefore be much larger.

We make two complementary points.
(i) For two \emph{independent} balanced binary labelings, small directional CDNV for both tasks implies their decision axes are nearly orthogonal.
(ii) There exist natural representation families where directional CDNV is simultaneously small for many tasks while CDNV is large, because most within-class variability lies in directions orthogonal to every task's decision axis.

{\bf Setup: two independent balanced multiclass labelings.\enspace} Let $z=f(x)\in\R^{d}$ with $\EE\|z\|_2^2<\infty$, and let $y^{(1)}\in[K_1],\, y^{(2)}\in[K_2]$ be two labelings of the same examples (where $[K]:=\{1,\dots,K\}$).
Assume both tasks are balanced and independent, i.e., for all $a \in [K_1]$ and $ b\in[K_2]$, we have
\begin{small}
\[
\Pr(y^{(1)}=a)=\frac1{K_1},\quad
\Pr(y^{(2)}=b)=\frac1{K_2},\quad
y^{(1)}\ \perp\ y^{(2)}.
\]
\end{small}
For each task $\ell\in\{1,2\}$ and class index $c\in[K_\ell]$, define $\mu^{(\ell)}_{c}:=\EE[z\mid y^{(\ell)}=c]$, and $\Sigma^{(\ell)}_{c}:=\Cov(z\mid y^{(\ell)}=c)$. For any pair of distinct classes $a\neq a'$ in task $1$, define the pairwise mean gap and (unit) decision axis $d^{(1)}_{aa'}:=\|\mu^{(1)}_{a}-\mu^{(1)}_{a'}\|_2$ and $u^{(1)}_{aa'}:=\frac{\mu^{(1)}_{a}-\mu^{(1)}_{a'}}{d^{(1)}_{aa'}}$, and similarly, for any $b\neq b'$ in task $2$, $d^{(2)}_{bb'}:=\|\mu^{(2)}_{b}-\mu^{(2)}_{b'}\|_2$ and $u^{(2)}_{bb'}:=\frac{\mu^{(2)}_{b}-\mu^{(2)}_{b'}}{d^{(2)}_{bb'}}$ (under the assume that $d^{(1)}_{aa'}>0$ and $d^{(2)}_{bb'}>0$).

For a fixed pair $(a,a')$ in task $1$, define the directional CDNV (maximized over all classes of task $1$) by
\[
\tilde V^{(1)}_{aa',c}
:=
\frac{\big(u^{(1)}_{aa'}\big)^\top \Sigma^{(1)}_{c}\,u^{(1)}_{aa'}}{\big(d^{(1)}_{aa'}\big)^2},
\quad
\tilde V^{(1)}_{aa'}
:=
\max_{c\in[K_1]} \tilde V^{(1)}_{aa',c},
\]
and analogously, for a fixed pair $(b,b')$ in task $2$, we define $\tilde V^{(2)}_{bb',c}$ and $\tilde V^{(2)}_{bb'}$

(For reference, the classical pairwise CDNV for task $1$ and pair $(a,a')$ is
\[
V^{(1)}_{aa'}:=\frac{\tr(\Sigma^{(1)}_{a})+\tr(\Sigma^{(1)}_{a'})}{\big(d^{(1)}_{aa'}\big)^2},
\]
and similarly for task $2$.)

{\bf Small directional CDNV forces near-orthogonality (multiclass version).\enspace}
The next proposition shows that small directional CDNV for two \emph{independent} multiclass tasks yields small overlap between any pairwise decision axes across the two tasks.

\begin{restatable}[Near-orthogonality from small directional CDNV]{proposition}{orthmulti}
\label{prop:near_orth_from_dircdnv_multiclass}
Assume $y^{(1)}\in[K_1]$ and $y^{(2)}\in[K_2]$ are balanced and independent.
Fix any $a\neq a'$ in $[K_1]$ and any $b\neq b'$ in $[K_2]$, and assume
$d^{(1)}_{aa'},d^{(2)}_{bb'}>0$.
Then
\begin{small}
\begin{equation*}
\left|\big(u^{(1)}_{aa'}\big)^\top u^{(2)}_{bb'}\right|
\;\le\;
\min\Bigg\{
\frac{d^{(1)}_{aa'}}{d^{(2)}_{bb'}}\sqrt{2K_2\tilde V^{(1)}_{aa'}},
\;
\frac{d^{(2)}_{bb'}}{d^{(1)}_{aa'}}\sqrt{2K_1\tilde V^{(2)}_{bb'}}
\Bigg\}.
\end{equation*}
\end{small}
\end{restatable}

Prop.~\ref{prop:near_orth_from_dircdnv_multiclass} shows that for two \emph{independent} balanced tasks, small directional CDNV implies that \emph{any} pairwise class-separating direction from task $1$ has small overlap with \emph{any} pairwise class-separating direction from task $2$. In other words, when within-class variability is small along the relevant margin direction for each task, the representation cannot reuse the same direction to separate classes for another independent task, except for a small residual overlap controlled by directional CDNV.

A useful consequence is that a single representation can support many independent tasks by allocating nearly orthogonal discriminative directions to different tasks, while still allowing substantial total within-class variance in directions that are orthogonal to these decision axes. This is precisely the regime where directional CDNV can be small for many tasks simultaneously, even if the classical CDNV (which sums variance across all directions) is large.

{\bf Why CDNV can still be large.\enspace}
Fix a task $\ell\in\{1,2\}$ and a pair of classes $c\neq c' \in [K_\ell]$, and write $u:=u^{(\ell)}_{cc'}$, $d:=d^{(\ell)}_{cc'}$, and $\Pi:=I-uu^\top$. For any class $r\in[K_\ell]$, the covariance decomposes as $\Sigma^{(\ell)}_{r}=(uu^\top)\Sigma^{(\ell)}_{r}(uu^\top)+(uu^\top)\Sigma^{(\ell)}_{r}\Pi+\Pi\Sigma^{(\ell)}_{r}(uu^\top)+\Pi\Sigma^{(\ell)}_{r}\Pi$. The directional CDNV for the pair $(c,c')$ only depends on the on-axis variance, namely $\tilde V^{(\ell)}_{cc',r}=\frac{u^\top \Sigma^{(\ell)}_{r}u}{d^2}$ and $\tilde V^{(\ell)}_{cc'}=\max_{r\in[K_\ell]}\tilde V^{(\ell)}_{cc',r}$, whereas the classical pairwise CDNV $V^{(\ell)}_{cc'}=\frac{\tr(\Sigma^{(\ell)}_{c})+\tr(\Sigma^{(\ell)}_{c'})}{d^2}$ counts variance in all directions. In particular, since $\tr(\Sigma^{(\ell)}_{r})=u^\top \Sigma^{(\ell)}_{r}u+\tr(\Pi\Sigma^{(\ell)}_{r}\Pi)$, the off-axis contribution $\tr(\Pi\Sigma^{(\ell)}_{r}\Pi)$ can be made arbitrarily large without changing $\tilde V^{(\ell)}_{cc',r}$ (and hence without changing $\tilde V^{(\ell)}_{cc'}$).

{\bf Many tasks with small dir-CDNV but large CDNV.\enspace} To illustrate the phenomenon transparently, consider $M\le d$ independent balanced binary tasks (a special case of the multiclass setup). Let $v_1,\dots,v_M\in\R^d$ be orthonormal, let $t^{(1)},\dots,t^{(M)}\in\{\pm1\}$ be i.i.d.\ balanced, and define $z=\sum_{\ell=1}^M \frac{\Delta_\ell}{2}\,t^{(\ell)}v_\ell+\eta+\xi$, where $\eta\in\operatorname{span}\{v_1,\dots,v_M\}^\perp$ and $(\eta,\xi)$ are independent of $(t^{(1)},\dots,t^{(M)})$.

For task $\ell$, taking conditional expectations gives $\mu^{(\ell)}_{+}-\mu^{(\ell)}_{-}=\Delta_\ell v_\ell$, so $v_\ell$ is exactly the decision axis and the pairwise mean gap is $\Delta_\ell$. Conditioning on $t^{(\ell)}$ does not affect the distribution of $\eta$, $\xi$, or $\{t^{(j)}:j\neq \ell\}$. Using $v_\ell^\top\eta=0$ and $v_\ell^\top v_j=0$ for $j\neq \ell$, we get $v_\ell^\top \Sigma^{(\ell)}_{s}v_\ell=v_\ell^\top \Cov(\xi)\,v_\ell$ for every $s\in\{\pm1\}$, and therefore $\tilde V^{(\ell)}=\frac{v_\ell^\top \Cov(\xi)\,v_\ell}{\Delta_\ell^2}$. Thus, if $\Cov(\xi)$ is small on $\mathcal U:=\operatorname{span}\{v_1,\dots,v_M\}$, then $\tilde V^{(\ell)}$ is small for every $\ell$.

By contrast, classical CDNV aggregates variance in all directions and therefore also counts the off-axis nuisance energy: $V^{(\ell)}=\frac{\tr(\Sigma^{(\ell)}_{+})+\tr(\Sigma^{(\ell)}_{-})}{\Delta_\ell^2}\ge \frac{2\,\tr(\Cov(\eta))}{\Delta_\ell^2}$. This can be made arbitrarily large by increasing $\tr(\Cov(\eta))$ while keeping all directional quantities fixed.

This example shows that an orthogonal-factor structure is a convenient \emph{sufficient} mechanism for obtaining small directional CDNV across many labelings, but it is not required for Proposition~\ref{prop:near_orth_from_dircdnv_multiclass}. The near-orthogonality conclusion there follows directly from task independence and small variance along the relevant pairwise decision axes.

\section{Experiments}

\begin{figure}[t]
  \centering
  \begin{tabular}{ccccc}
    \includegraphics[width=0.45\linewidth]{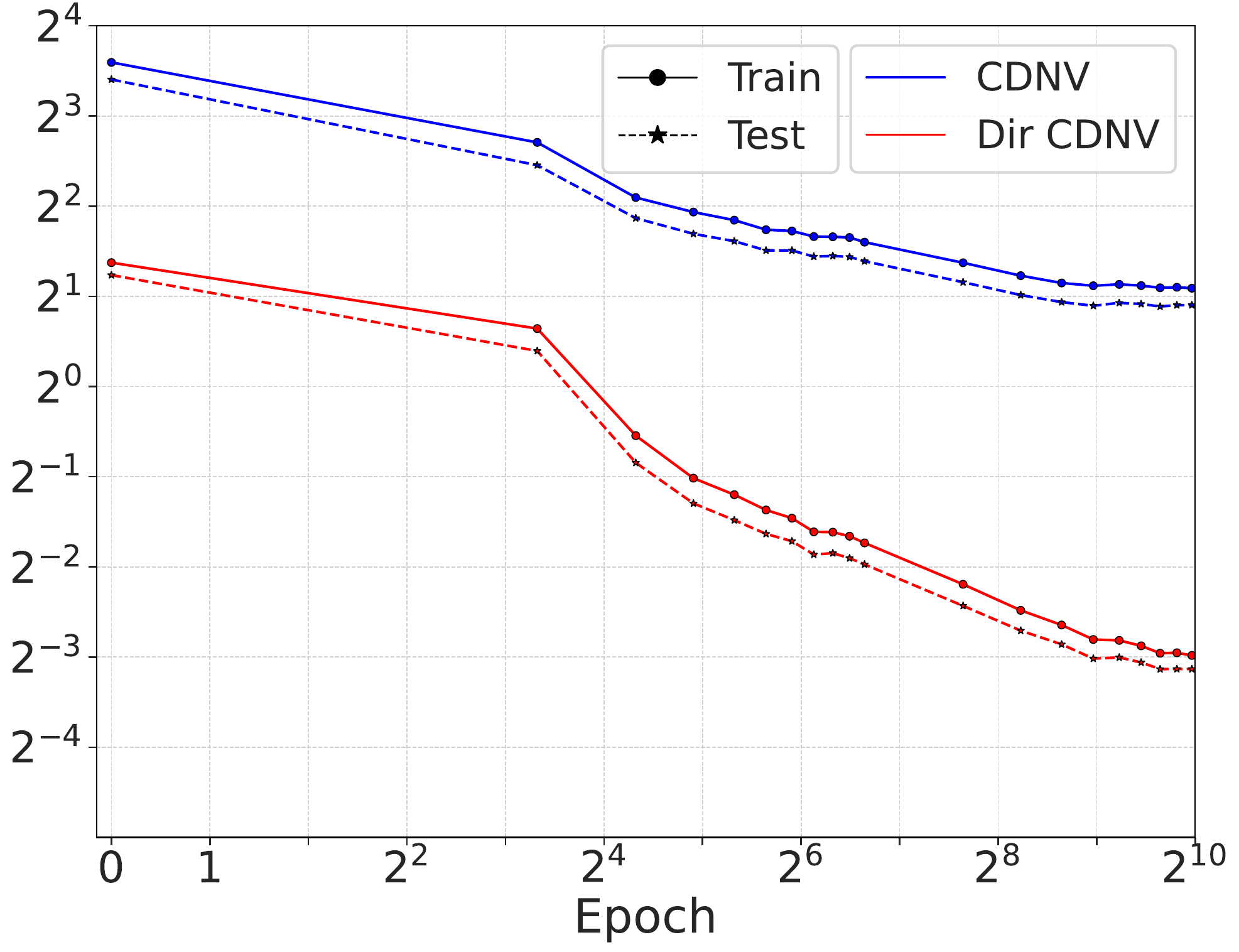} &
    \includegraphics[width=0.45\linewidth]{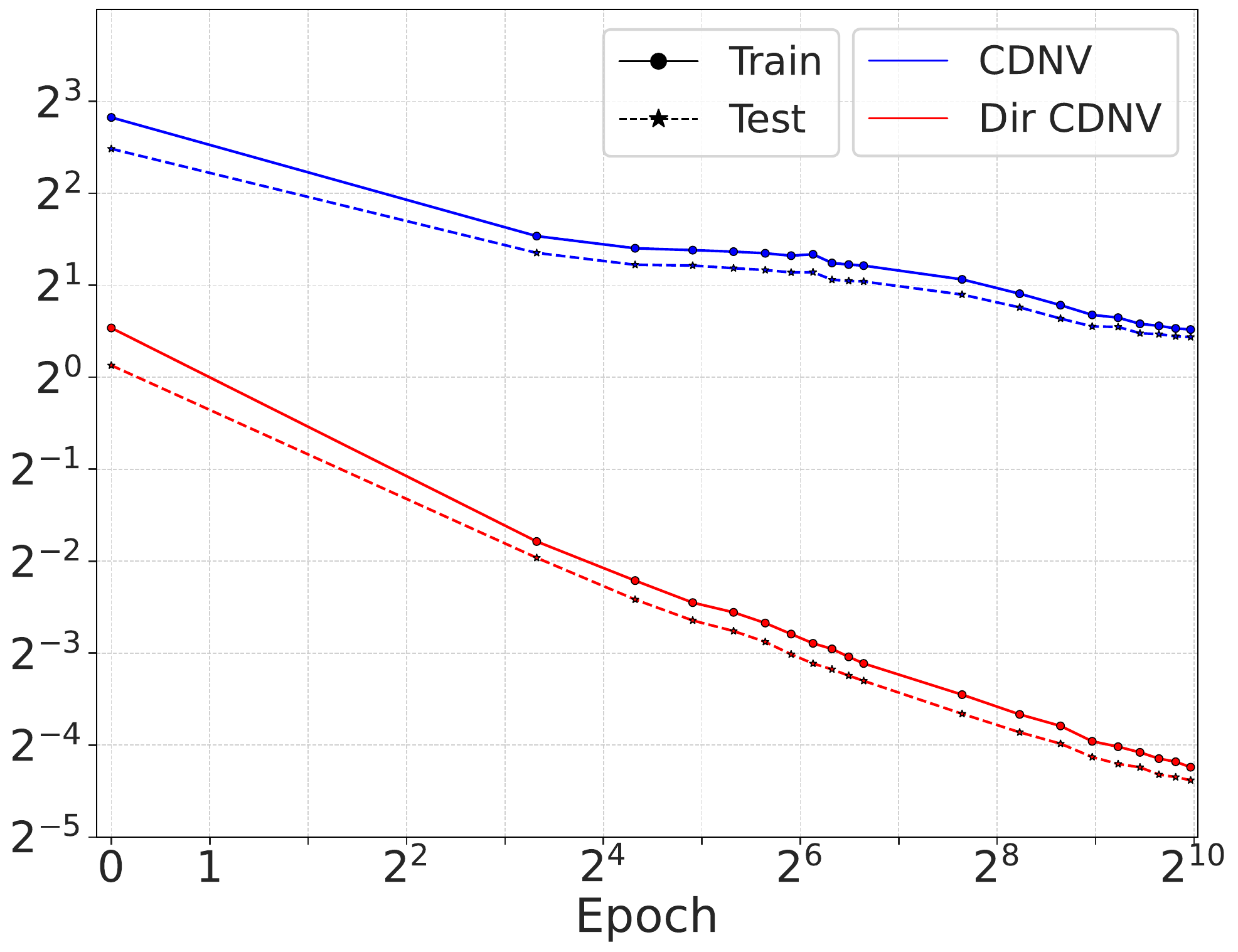} \\
    {\small {\bf MAE (ViT-B/16)}} & {\small {\bf SimCLR (ResNet-50)}} \\
    \includegraphics[width=0.45\linewidth]{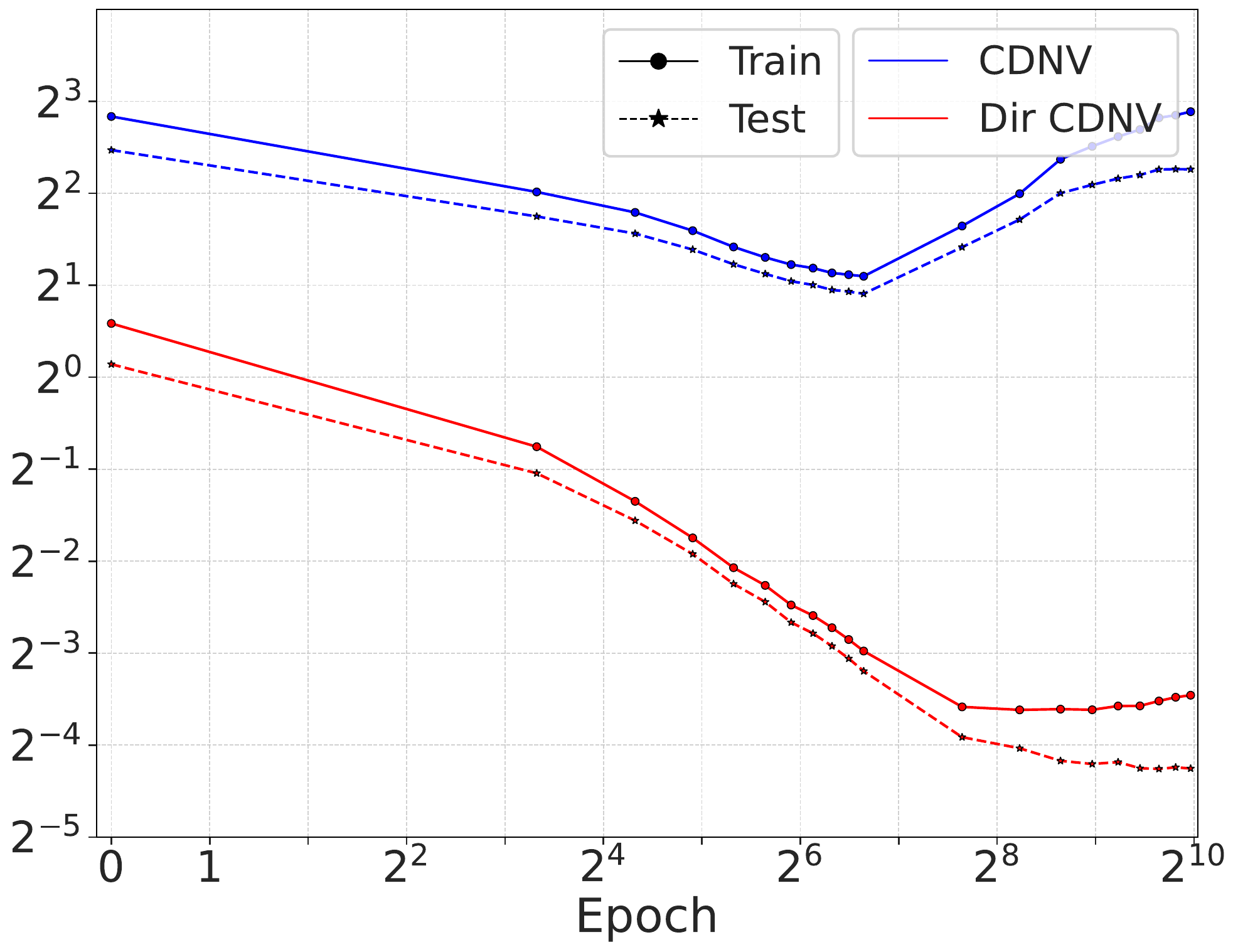}
 &
 \includegraphics[width=0.45\linewidth]{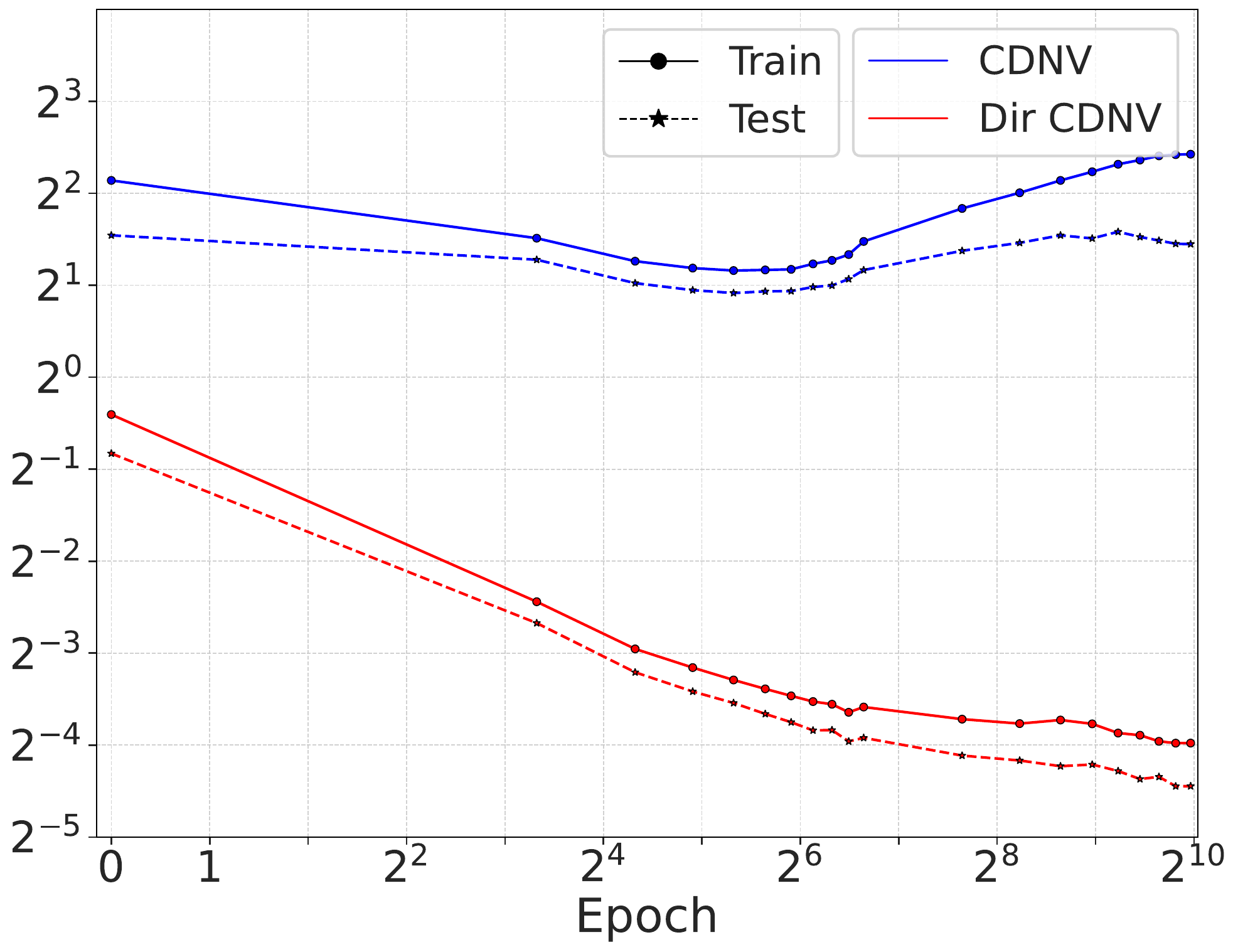} \\
 {\small {\bf DINO-v2 (ViT-B/16)}} & {\small {\bf VICReg (ResNet-50)}} \\

  \end{tabular}
\caption{\textbf{Decision-axis collapse emerges during SSL training.}
We track both CDNV and directional CDNV on the training and test sets. Directional CDNV decreases much more than CDNV, indicating that SSL primarily tightens class geometry along separating directions even when overall within-class variability is large.}
\label{fig:dir}
\end{figure}

\begin{figure}[t]
  \centering
  \begin{tabular}{ccccc}
    \includegraphics[width=0.45\linewidth]{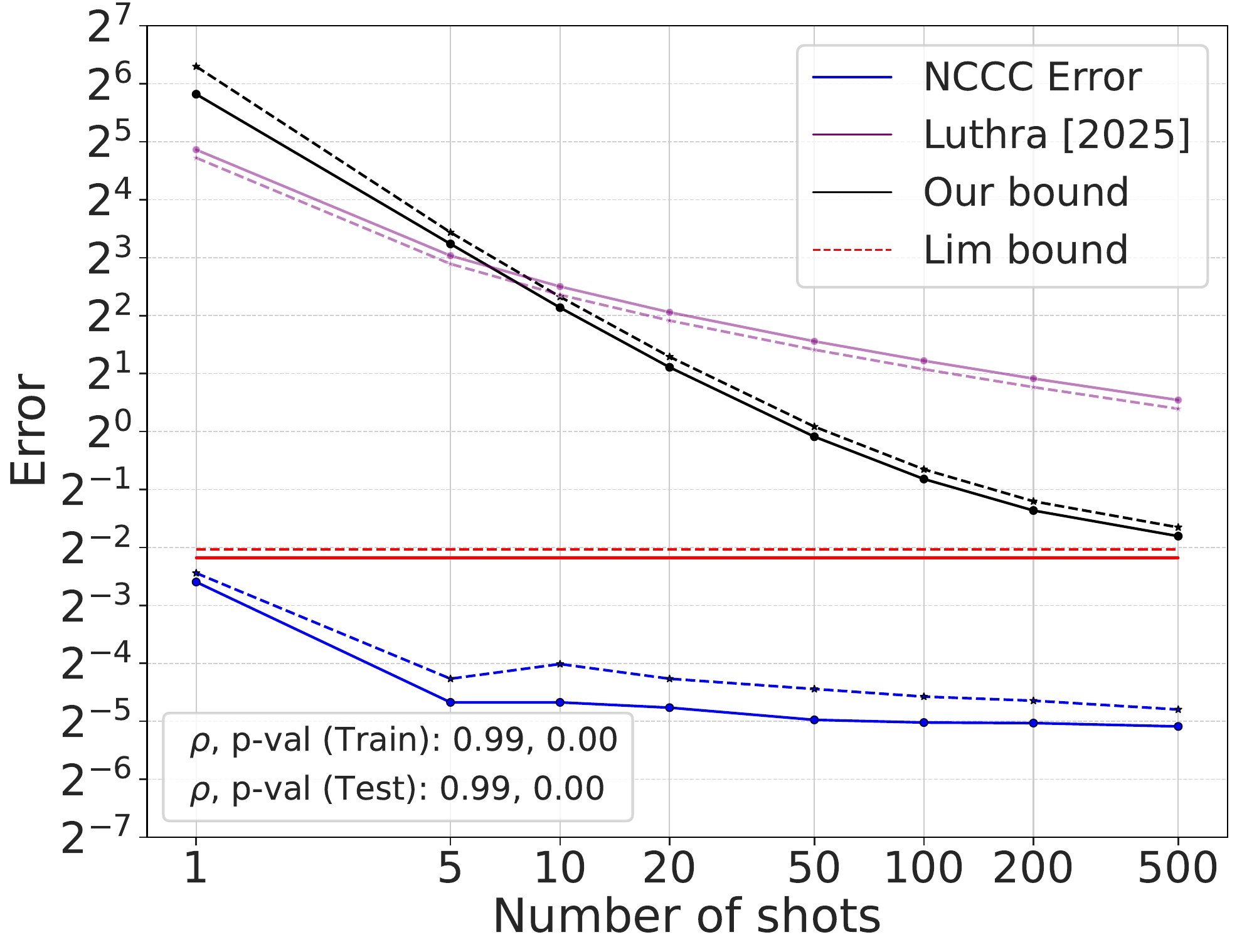} &
    \includegraphics[width=0.45\linewidth]{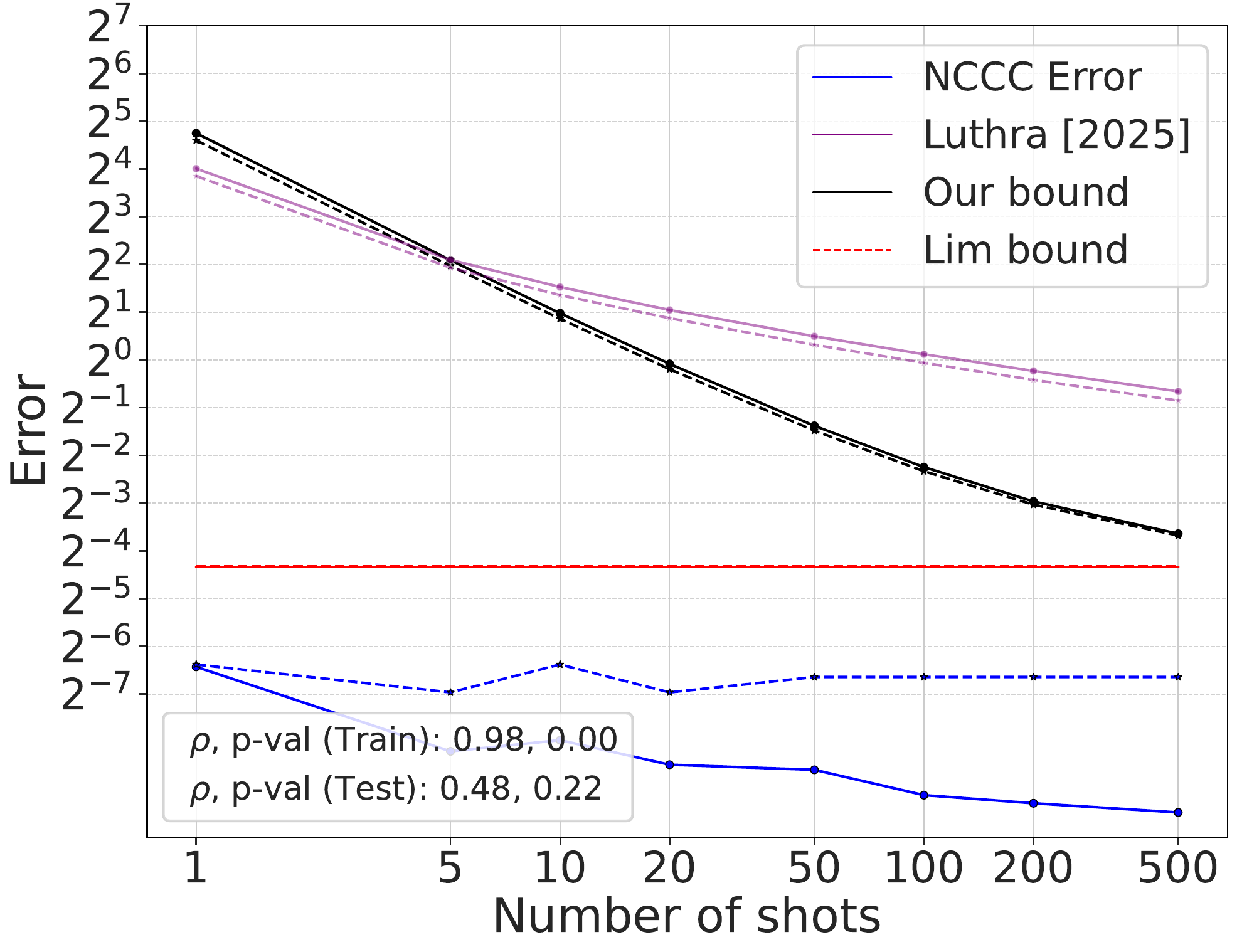} \\
    {\small {\bf MAE (ViT-B/16)}} & {\small {\bf CLIP (ViT-B/16)}} \\
    \includegraphics[width=0.45\linewidth]{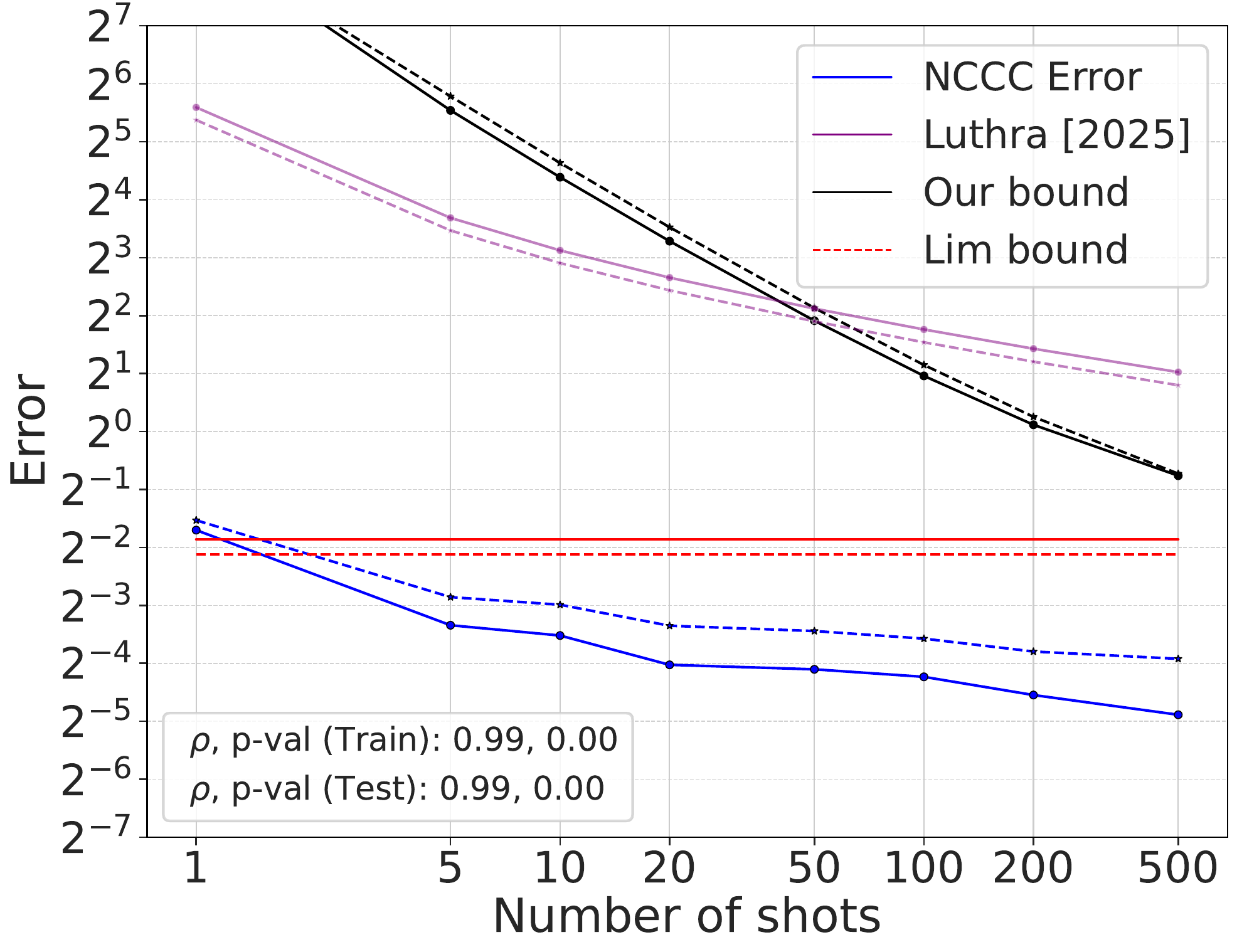}
 &
 \includegraphics[width=0.45\linewidth]{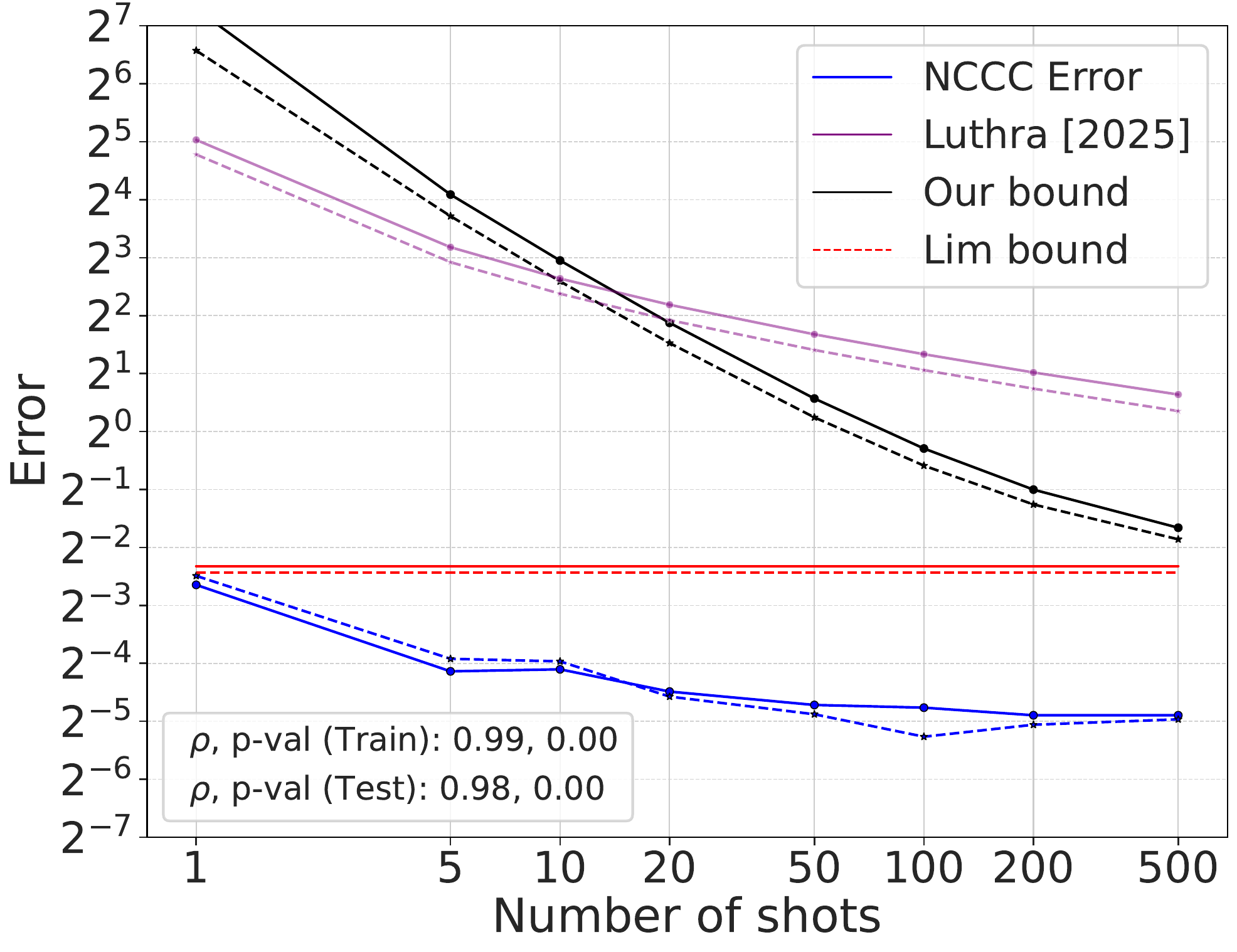} \\
 {\small {\bf I-Jepa (ViT-B/16)}} & {\small {\bf SigLip (ViT-B/16) }} \\
     \includegraphics[width=0.45\linewidth]{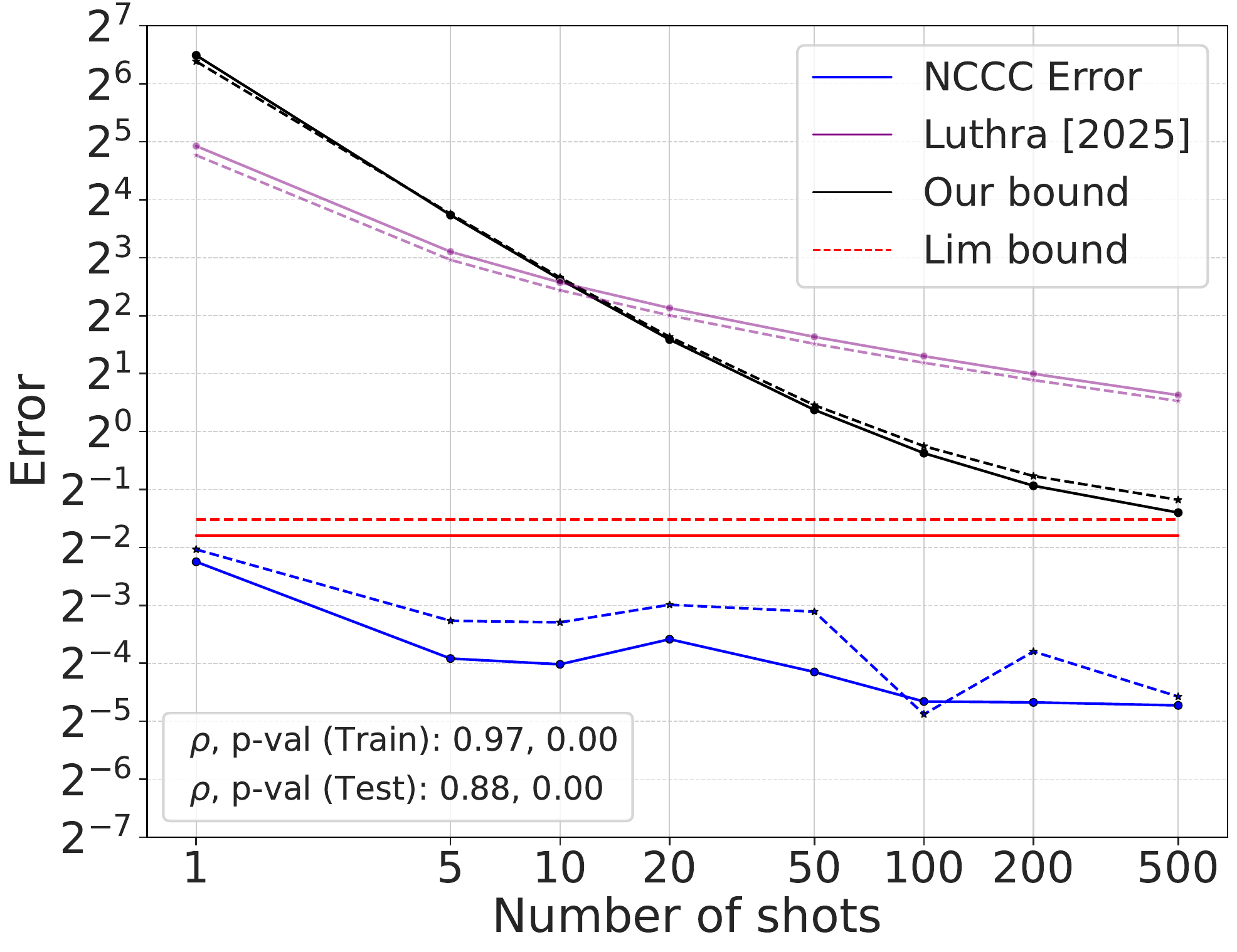} &
    \includegraphics[width=0.45\linewidth]{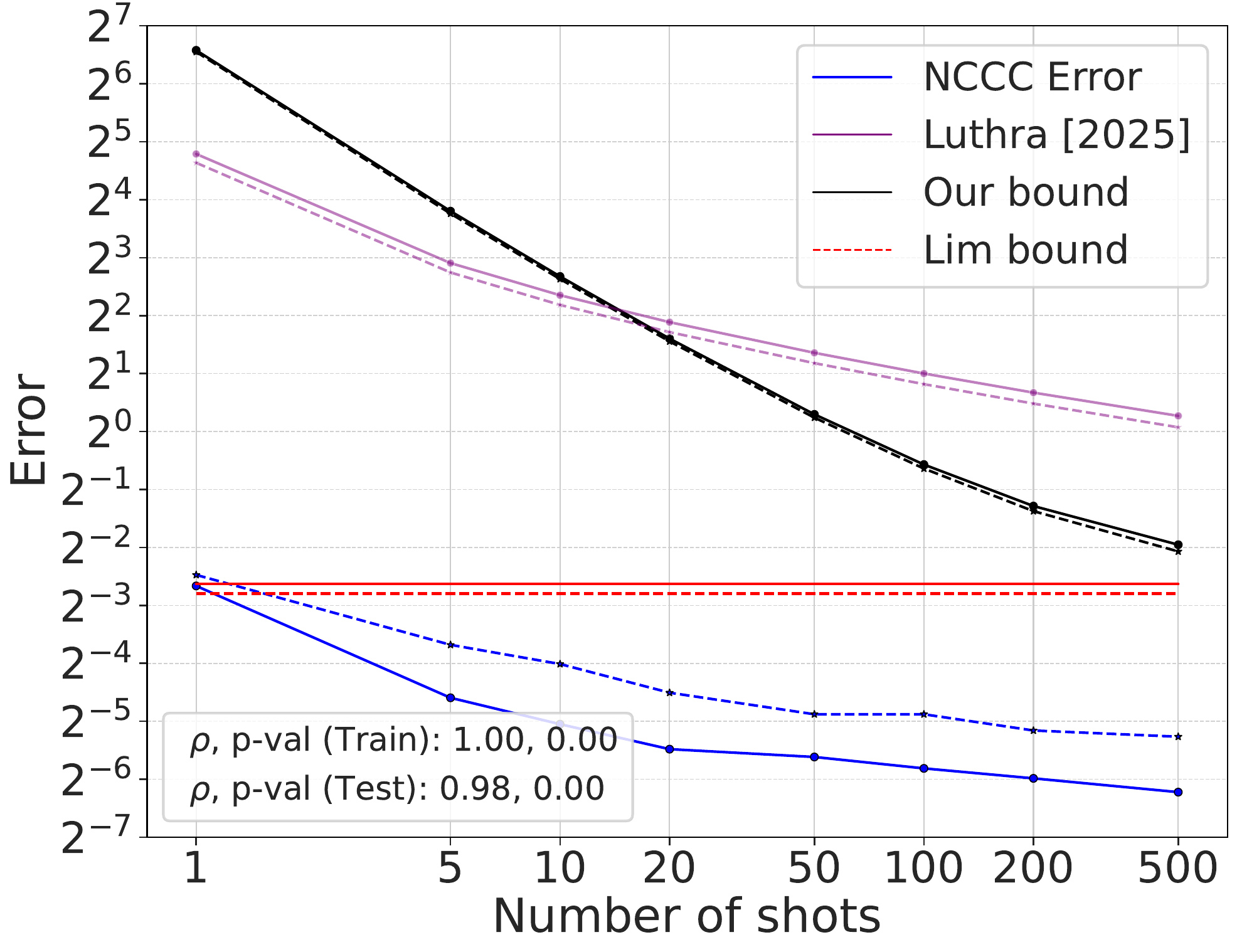} \\
 {\small {\bf SimCLR (ResNet-50)}} & {\small {\bf VICReg (ResNet-50)}} \\

  \end{tabular}
\caption{\textbf{Decision-axis variance yields informative few-shot certificates in SSL.} We plot few-shot NCC test error versus the number of shots per class, $m$, for several pretrained SSL encoders, together with certified upper bounds from our analysis. We compare our finite-$m$ bound to the directional-only $m\to\infty$ limit and to the bound of~\cite{luthra2025selfsupervisedcontrastivelearningapproximately}.}
\label{fig:bounds}
\end{figure}

\begin{figure}[t]
  \centering
  \begin{tabular}{ccccc}
    \includegraphics[width=0.45\linewidth]{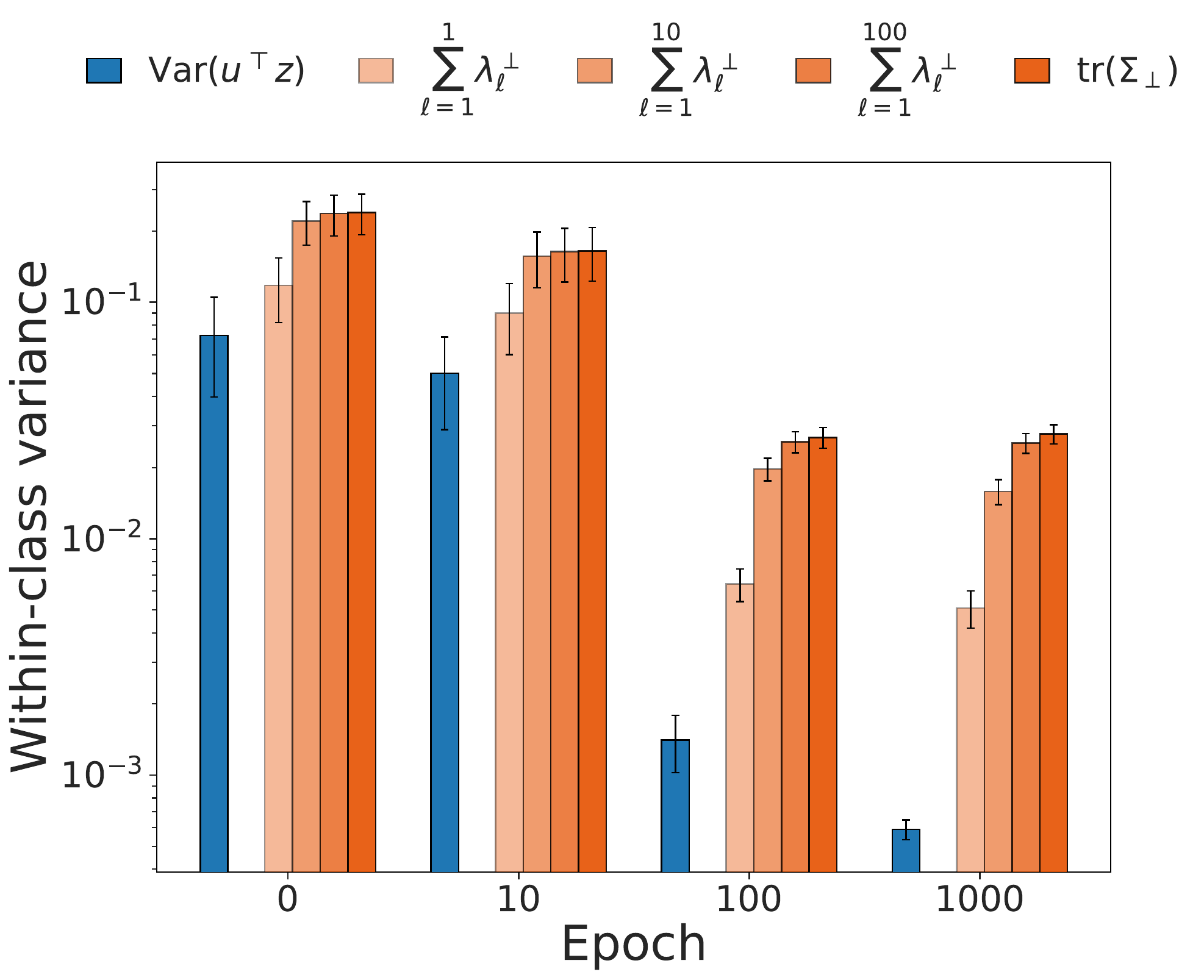} &
\includegraphics[width=0.45\linewidth]{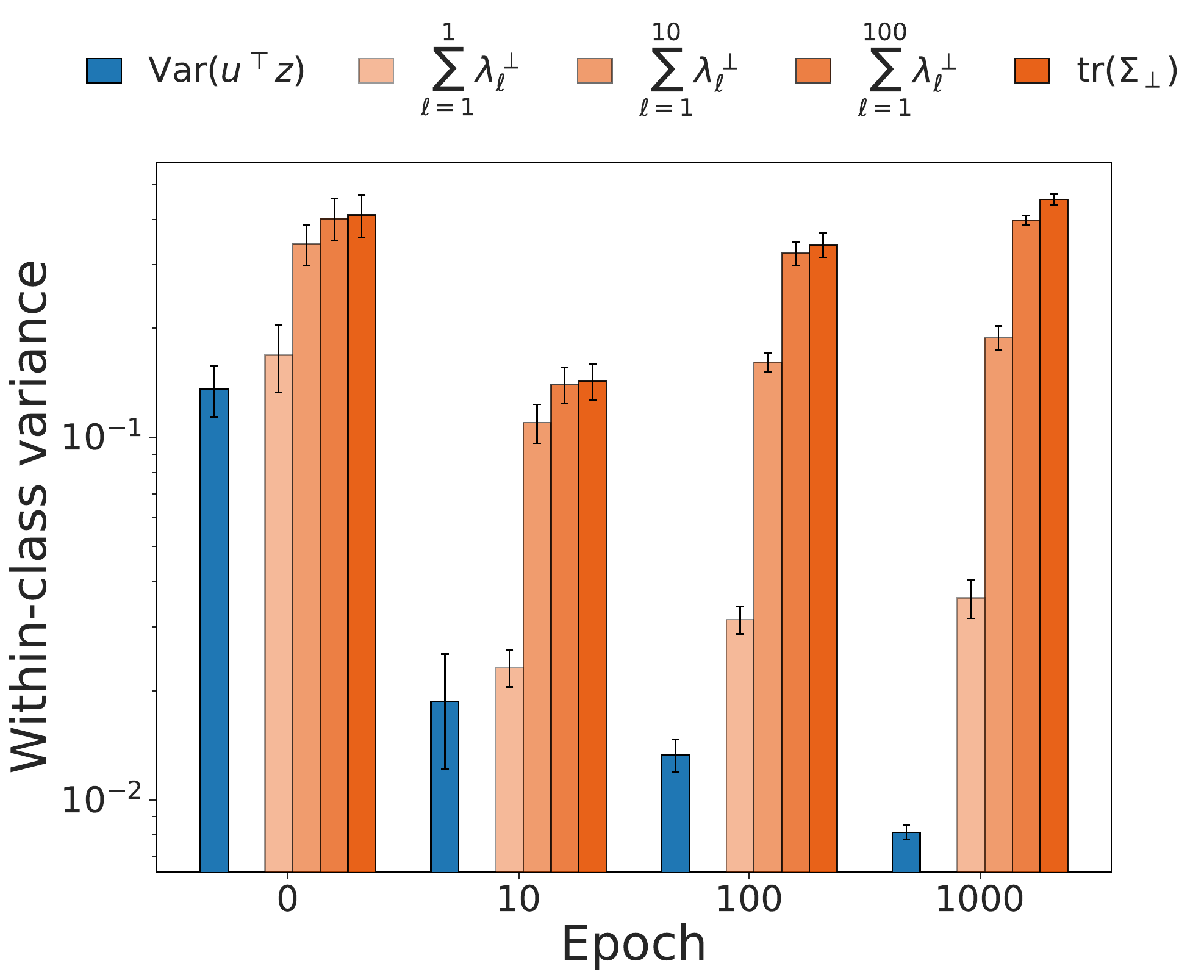} \\
{\small\bf MAE (ViT-B/16)} & {\small\bf SimCLR (ResNet-50)} \\

    \includegraphics[width=0.45\linewidth]{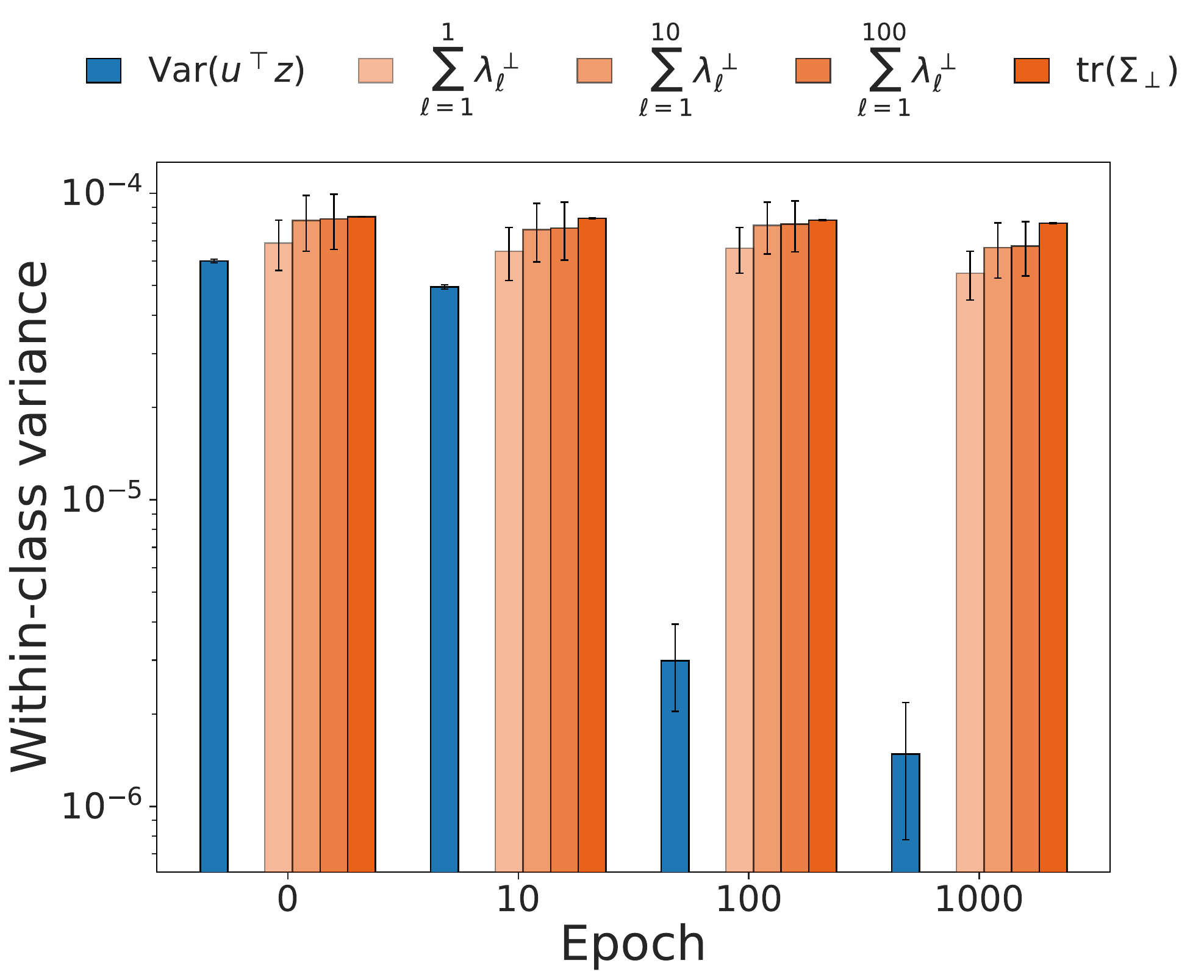}
 &
 \includegraphics[width=0.45\linewidth]{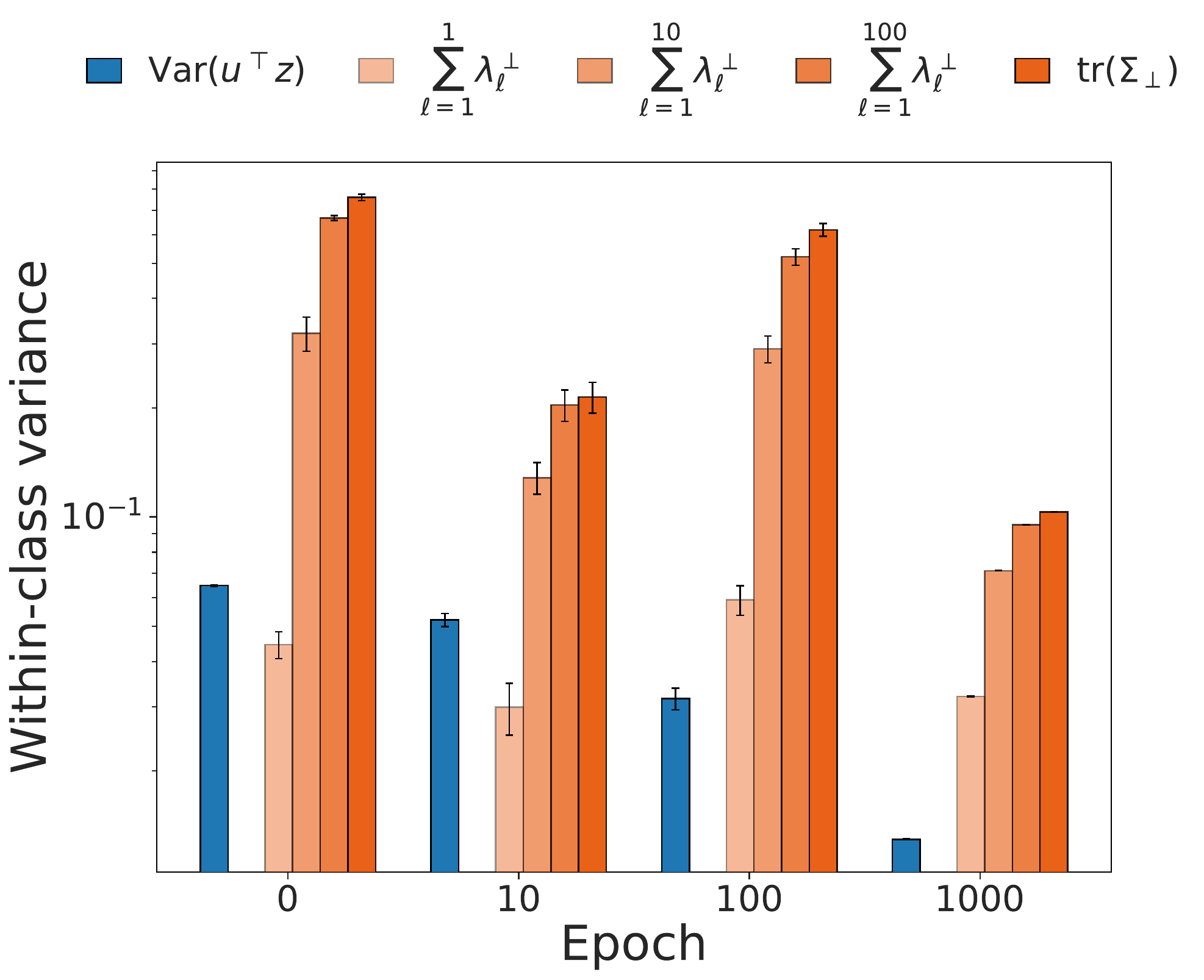} \\
 {\small {\bf DINO-v2 (ViT-B/16)}} & {\small {\bf VicReg (ResNet-50)}} \\

  \end{tabular}
\caption{\textbf{Decision-axis variance collapses while orthogonal variance remains large.} 
Within-class variance decomposed into decision-axis and orthogonal components shows rapid collapse along the task-relevant direction despite persistently large orthogonal variance.
}
\label{fig:variance-decomp}
\end{figure}


\subsection{Settings}

In this section we describe the experimental settings. Please refer to App.~\ref{app:details} for the full training, optimization, and augmentation details.

{\bf Datasets and augmentations. \enspace} We conduct all our experiments on the standard image classification dataset -- mini-ImageNet~\citep{NIPS2016_90e13578} which is a subset of ImageNet-1K~\cite{5206848}. 

{\bf Methods, optimizers and augmentations. \enspace}
We study a range of self-supervised paradigms, including contrastive learning (SimCLR~\citep{pmlr-v119-chen20j}, VICReg~\citep{bardes2022vicreg}), masked modeling (I-JEPA~\citep{DBLP:conf/cvpr/AssranDMBVRLB23}, MAE~\citep{he2022masked}), distillation-based methods (DINO-v2~\citep{oquab2024dinov2learningrobustvisual}), and multimodal pretraining (CLIP~\citep{pmlr-v139-radford21a}, and SigLIP~\citep{Zhai_2023_ICCV}.
For learning-dynamics experiments, we train SimCLR, VICReg, MAE, and DINO-v2 from scratch on mini-ImageNet. Across experiments, we use ResNet-18/ResNet-50 backbones~\citep{He_2016_CVPR} (with width multiplier 2) and a ViT-Base (ViT-B/16) backbone~\citep{dosovitskiy2021imageworth16x16words}, and follow the standard augmentation pipelines associated with each method.

\subsection{Results}

{\bf Directional-CDNV vs. CDNV.\enspace} In Fig.~\ref{fig:dir} we test whether SSL training preferentially minimizes directional CDNV.
To do so, we pretrain MAE, SimCLR, DINO-v2, and VICReg for 1000 epochs and track both standard CDNV and directional CDNV on the training and test sets.
As shown in Fig.~\ref{fig:dir}, directional CDNV drops dramatically over training, from roughly \(2^{-1}\!-\!2^{1}\) down to about \(2^{-3}\!-\!2^{-5}\).
In contrast, standard CDNV decreases only modestly and in some cases even increases transiently. These trends suggest that full within-class variance collapse is not the dominant effect in SSL. Instead, SSL mainly suppresses variability along class-separating directions, and this directional form of collapse appears consistently across a broad range of methods.

{\bf Validating Thm.~\ref{thm:ncc-full-optimized} with varying $m$.\enspace} We use frozen, off-the-shelf vision encoders, pretrained on IM-1K for vision-only SSL methods, on 400M image-text pairs for CLIP, and on WebLI~\citep{chen2023palijointlyscaledmultilinguallanguageimage} for SigLIP. We then evaluate nearest-class-centroid (NCC) classification on mini-ImageNet without finetuning, using binary tasks ($C=2$) formed by randomly selecting two classes. For each shot count $m$, we estimate class centroids from $m$ randomly sampled training examples per class and report test accuracy; we average over both the random choice of the two classes and the random draw of the $m$ shots, repeating this procedure five times.

As shown in Fig.~\ref{fig:bounds}, our finite-$m$ certificates are informative at practical shot counts and closely track the observed few-shot error. The finite-$m$ curve converges (by construction) to the directional-only $m\to\infty$ limit shown in the figure. Importantly, the bound is non-vacuous in this regime: for moderate $m$, our certified error drops below the $0.5$ chance threshold, whereas the bound of~\cite{luthra2025selfsupervisedcontrastivelearningapproximately} remains above $0.5$ here and is therefore vacuous.

{\bf Decision-axis vs orthogonal variance.\enspace} To examine how within-class variance evolves during self-supervised training, we analyze the geometry of learned representations at the level of class pairs. For each model and training epoch, we randomly select 20 class pairs $(i,j)$ and define the corresponding decision axis $u_{ij}$ as the normalized difference between class means. Using this axis and the within-class covariance matrix of representations \(\Sigma_{ij} \;=\; \frac{1}{N_{ij}} \sum_{n:\,y_n\in\{i,j\}} (z_n-\mu_{y_n}) (z_n-\mu_{y_n})^\top\), we decompose it into a component along the decision axis as \(u_{ij}^\top \Sigma_{ij} u_{ij}\). We compute the corresponding orthogonal covariance matrix as \(P_{ij}^{\perp}\,\Sigma_{ij}\,P_{ij}^{\perp}\) where $P_{ij}^{\perp} = I - u_{ij}u_{ij}^\top$. We then perform an eigen-decomposition of the orthogonal covariance matrix and compute the cumulative variance captured by its top $k$ principle directions. 

In Fig.~\ref{fig:variance-decomp}, we report variance along the decision axis, cumulative orthogonal variance for \(k \in \{1, 10, 100\}\), and the total orthogonal variance, averaged across the selected pairs. Across all methods, variance along the decision axis collapses rapdily with training, while substantial variance persists in the orthogonal directions which are irrelevant for task-specific downstream classification.  

\begin{figure}[t]
    \centering
    \begin{tabular}{cc}
         \includegraphics[width=0.49\linewidth]{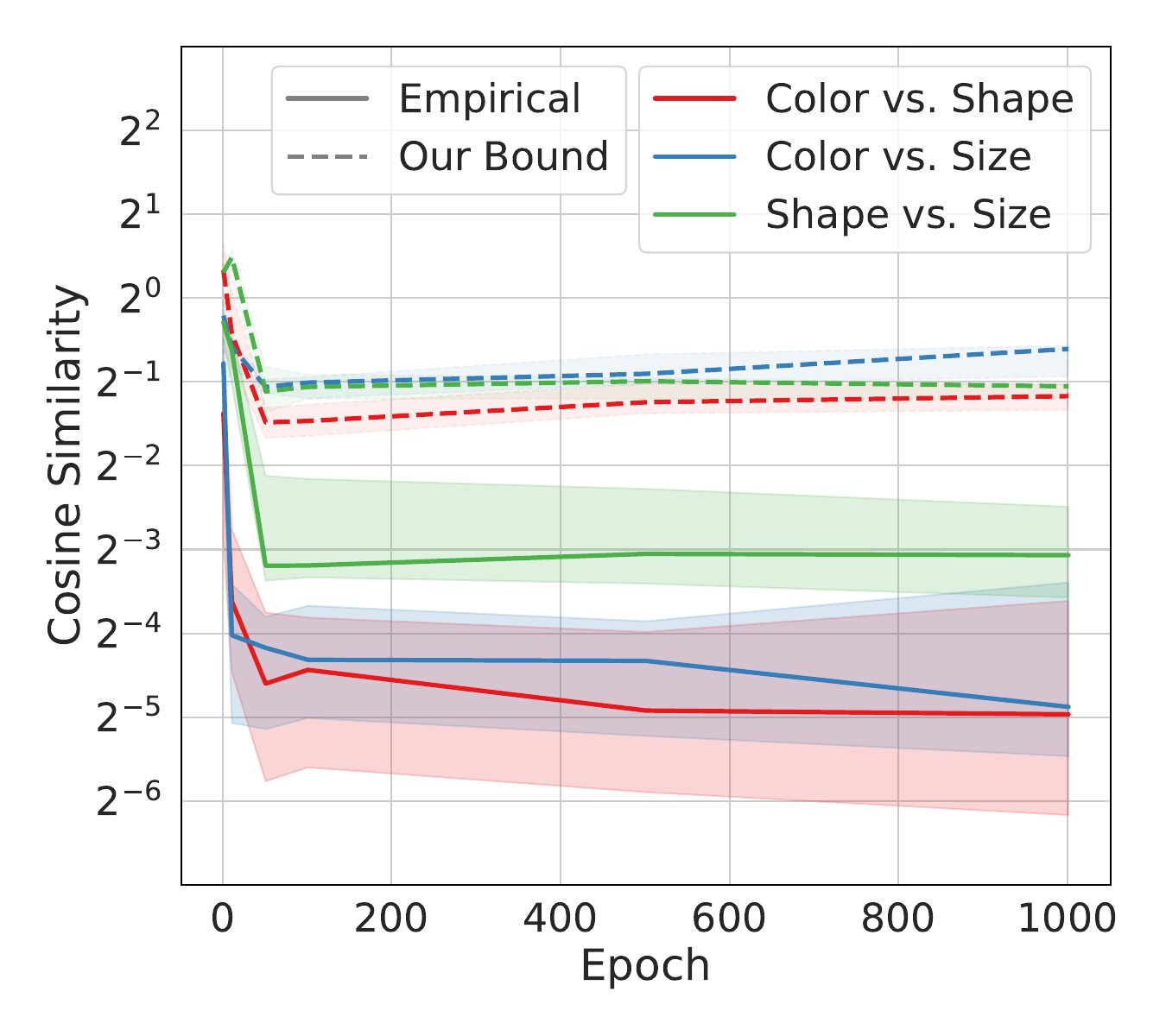} & 
         \includegraphics[width=0.49\linewidth]{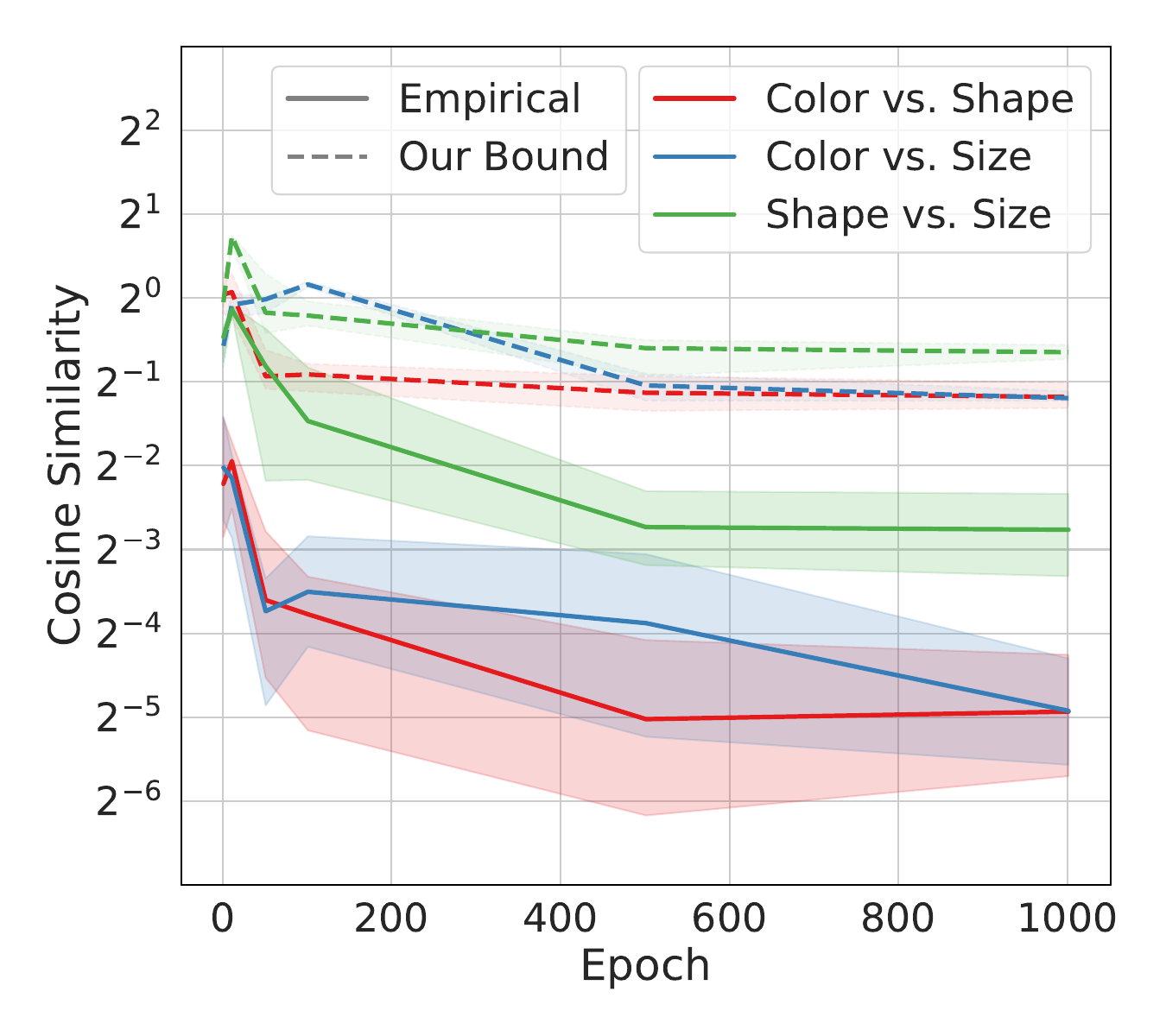} \\
         {\small {\bf (a) SimCLR (ResNet-18)}} & {\small {\bf (b) VicReg (ResNet-50)}}
    \end{tabular}
    \caption{{\bf Multitask decision-axis orthogonalization in SSL.} Median absolute cosine similarity (25–75\% bands) between decision axes from different semantic labelings during training; alignments decay toward zero and are upper-bounded by our theory (dashed).}
    \label{fig:multitask_orth}
\end{figure}

{\bf Multitask orthogonal representation decomposition.\enspace} To study how representations learned via self-supervised learning organize multiple semantic labelings in high-dimensional feature space, we design a controlled synthetic experiment (Fig.~\ref{fig:multitask_orth}). Each sample is a \((64 \times64)\) image generated from a combination of independent factors of variation, including \emph{color} ($L_1$), \emph{background style} ($L_2$), \emph{shape} ($L_3$), and \emph{shape size} ($L_4$), yielding multiple valid labelings over the same set of inputs.
We pre-train a ResNet-18 encoder from scratch on this dataset using self-supervised objectives. At different stages of training, we extract the learned feature representations and compute class means for each labeling $L_n$, denoted by $\mu^{(L_n)}_{y}$. For each labeling, we define \emph{decision axes} as normalized differences between class means: $u^{(L_n)}_{y_i y_j}
= (\mu^{(L_n)}_{y_i} - \mu^{(L_n)}_{y_j})/
\| \mu^{(L_n)}_{y_i} - \mu^{(L_n)}_{y_j} \|$, where $y_i \neq y_j$. These directions characterize the linear decision geometry associated with each labeling. To quantify the interaction between different labelings, we compute the cosine similarity between decision axes drawn from \emph{distinct} labelings, i.e.,
$\lvert \langle u^{(L_a)}_{ij}, u^{(L_b)}_{kl} \rangle \rvert$ for $L_a \neq L_b$, and aggregate statistics across all class pairs.

In Fig.~\ref{fig:multitask_orth}, we report the median absolute cosine similarity together with interquartile (25--75\%) bands over training. At random initialization, decision axes corresponding to different labelings exhibit substantial alignment, indicating entangled representations. As training progresses, these cosine similarities rapidly decrease and stabilize near zero, demonstrating that the learned representation decomposes into \emph{approximately orthogonal semantic subspaces} corresponding to different labelings. Importantly, this orthogonalization occurs simultaneously for multiple independent labelings, indicating that SSL representations do not collapse to a single discriminative direction but instead support multiple, low-interference decision geometries.

\section*{Impact Statement}

This work is theoretical and raises no direct ethical, safety, or environmental concerns. We study conditions under which self-supervised learning methods adapt to downstream tasks, introduce representation properties that help explain when this is possible, and provide empirical evidence supporting these claims. These results inform the high-level design and evaluation of future self-supervised learning algorithms, with very limited direct societal impact.

\bibliography{refs}

@inproceedings{
bardes2022vicreg,
title={{VICR}eg: Variance-Invariance-Covariance Regularization for Self-Supervised Learning},
author={Adrien Bardes and Jean Ponce and Yann LeCun},
booktitle={International Conference on Learning Representations},
year={2022},
url={https://openreview.net/forum?id=xm6YD62D1Ub}
}

@inproceedings{DBLP:conf/icml/YuanWQDZZY22,
  author       = {Zhuoning Yuan and
                  Yuexin Wu and
                  Zi{-}Hao Qiu and
                  Xianzhi Du and
                  Lijun Zhang and
                  Denny Zhou and
                  Tianbao Yang},
  editor       = {Kamalika Chaudhuri and
                  Stefanie Jegelka and
                  Le Song and
                  Csaba Szepesv{\'{a}}ri and
                  Gang Niu and
                  Sivan Sabato},
  title        = {Provable Stochastic Optimization for Global Contrastive Learning:
                  Small Batch Does Not Harm Performance},
  booktitle    = {International Conference on Machine Learning, {ICML} 2022, 17-23 July
                  2022, Baltimore, Maryland, {USA}},
  series       = {Proceedings of Machine Learning Research},
  volume       = {162},
  pages        = {25760--25782},
  publisher    = {{PMLR}},
  year         = {2022},
  url          = {https://proceedings.mlr.press/v162/yuan22b.html},
  bibsource    = {dblp computer science bibliography, https://dblp.org}
}

@InProceedings{pmlr-v119-chen20j,
  title = 	 {A Simple Framework for Contrastive Learning of Visual Representations},
  author =       {Chen, Ting and Kornblith, Simon and Norouzi, Mohammad and Hinton, Geoffrey},
  booktitle = 	 {Proceedings of the 37th International Conference on Machine Learning},
  pages = 	 {1597--1607},
  year = 	 {2020},
  editor = 	 {III, Hal Daumé and Singh, Aarti},
  volume = 	 {119},
  series = 	 {Proceedings of Machine Learning Research},
  month = 	 {13--18 Jul},
  publisher =    {PMLR},
  url = 	 {https://proceedings.mlr.press/v119/chen20j.html},
}

@InProceedings{pmlr-v162-awasthi22b,
  title = 	 {Do More Negative Samples Necessarily Hurt In Contrastive Learning?},
  author =       {Awasthi, Pranjal and Dikkala, Nishanth and Kamath, Pritish},
  booktitle = 	 {Proceedings of the 39th International Conference on Machine Learning},
  pages = 	 {1101--1116},
  year = 	 {2022},
  editor = 	 {Chaudhuri, Kamalika and Jegelka, Stefanie and Song, Le and Szepesvari, Csaba and Niu, Gang and Sabato, Sivan},
  volume = 	 {162},
  series = 	 {Proceedings of Machine Learning Research},
  month = 	 {17--23 Jul},
  publisher =    {PMLR},
  url = 	 {https://proceedings.mlr.press/v162/awasthi22b.html},
}

@InProceedings{pmlr-v162-bao22e,
  title = 	 {On the Surrogate Gap between Contrastive and Supervised Losses},
  author =       {Bao, Han and Nagano, Yoshihiro and Nozawa, Kento},
  booktitle = 	 {Proceedings of the 39th International Conference on Machine Learning},
  pages = 	 {1585--1606},
  year = 	 {2022},
  editor = 	 {Chaudhuri, Kamalika and Jegelka, Stefanie and Song, Le and Szepesvari, Csaba and Niu, Gang and Sabato, Sivan},
  volume = 	 {162},
  series = 	 {Proceedings of Machine Learning Research},
  month = 	 {17--23 Jul},
  publisher =    {PMLR},
  url = 	 {https://proceedings.mlr.press/v162/bao22e.html},
}

@misc{feigin2025theoreticalcharacterizationoptimaldata,
      title={A Theoretical Characterization of Optimal Data Augmentations in Self-Supervised Learning}, 
      author={Shlomo Libo Feigin and Maximilian Fleissner and Debarghya Ghoshdastidar},
      year={2025},
      eprint={2411.01767},
      archivePrefix={arXiv},
      primaryClass={cs.LG},
      url={https://arxiv.org/abs/2411.01767}, 
}

@inproceedings{
tian2023understanding,
title={Understanding the Role of Nonlinearity in Training Dynamics of Contrastive Learning},
author={Yuandong Tian},
booktitle={The Eleventh International Conference on Learning Representations },
year={2023},
url={https://openreview.net/forum?id=s130rTE3U_X}
}

@InProceedings{pmlr-v139-wen21c,
  title = 	 {Toward Understanding the Feature Learning Process of Self-supervised Contrastive Learning},
  author =       {Wen, Zixin and Li, Yuanzhi},
  booktitle = 	 {Proceedings of the 38th International Conference on Machine Learning},
  pages = 	 {11112--11122},
  year = 	 {2021},
  editor = 	 {Meila, Marina and Zhang, Tong},
  volume = 	 {139},
  series = 	 {Proceedings of Machine Learning Research},
  month = 	 {18--24 Jul},
  publisher =    {PMLR},
  url = 	 {https://proceedings.mlr.press/v139/wen21c.html},
}

@inproceedings{NEURIPS2022_7b5c9cc0,
 author = {Tian, Yuandong},
 booktitle = {Advances in Neural Information Processing Systems},
 editor = {S. Koyejo and S. Mohamed and A. Agarwal and D. Belgrave and K. Cho and A. Oh},
 pages = {19511--19522},
 publisher = {Curran Associates, Inc.},
 title = {Understanding Deep Contrastive Learning via Coordinate-wise Optimization},
 url = {https://proceedings.neurips.cc/paper_files/paper/2022/file/7b5c9cc08960df40615c1d858961eb8b-Paper-Conference.pdf},
 volume = {35},
 year = {2022}
}

@article{10.5555/3648699.3649029,
author = {Ji, Wenlong and Deng, Zhun and Nakada, Ryumei and Zou, James and Zhang, Linjun},
title = {The power of contrast for feature learning: a theoretical analysis},
year = {2023},
issue_date = {January 2023},
publisher = {JMLR.org},
volume = {24},
number = {1},
issn = {1532-4435},
journal = {J. Mach. Learn. Res.},
month = jan,
articleno = {330},
numpages = {78}
}

@article{
galanti2023comparative,
title={Comparative Generalization Bounds for Deep Neural Networks},
author={Tomer Galanti and Liane Galanti and Ido Ben-Shaul},
journal={Transactions on Machine Learning Research},
issn={2835-8856},
year={2023},
url={https://openreview.net/forum?id=162TqkUNPO},
note={}
}

@inproceedings{NEURIPS2021_02e656ad,
 author = {Lee, Jason D and Lei, Qi and Saunshi, Nikunj and ZHUO, JIACHENG},
 booktitle = {Advances in Neural Information Processing Systems},
 editor = {M. Ranzato and A. Beygelzimer and Y. Dauphin and P.S. Liang and J. Wortman Vaughan},
 pages = {309--323},
 publisher = {Curran Associates, Inc.},
 title = {Predicting What You Already Know Helps: Provable Self-Supervised Learning},
 url = {https://proceedings.neurips.cc/paper_files/paper/2021/file/02e656adee09f8394b402d9958389b7d-Paper.pdf},
 volume = {34},
 year = {2021}
}

@inproceedings{
garrido2023on,
title={On the duality between contrastive and non-contrastive self-supervised learning},
author={Quentin Garrido and Yubei Chen and Adrien Bardes and Laurent Najman and Yann LeCun},
booktitle={The Eleventh International Conference on Learning Representations },
year={2023},
url={https://openreview.net/forum?id=kDEL91Dufpa}
}

@inproceedings{
haochen2023a,
title={A theoretical study of inductive biases in contrastive learning},
author={Jeff Z. HaoChen and Tengyu Ma},
booktitle={The Eleventh International Conference on Learning Representations },
year={2023},
url={https://openreview.net/forum?id=AuEgNlEAmed}
}

@inproceedings{10.5555/3600270.3602219,
author = {HaoChen, Jeff Z. and Wei, Colin and Kumar, Ananya and Ma, Tengyu},
title = {Beyond separability: analyzing the linear transferability of contrastive representations to related subpopulations},
year = {2022},
isbn = {9781713871088},
publisher = {Curran Associates Inc.},
address = {Red Hook, NY, USA},
booktitle = {Proceedings of the 36th International Conference on Neural Information Processing Systems},
articleno = {1949},
numpages = {14},
location = {New Orleans, LA, USA},
series = {NIPS '22}
}

@InProceedings{pmlr-v162-shen22d,
  title = 	 {Connect, Not Collapse: Explaining Contrastive Learning for Unsupervised Domain Adaptation},
  author =       {Shen, Kendrick and Jones, Robbie M and Kumar, Ananya and Xie, Sang Michael and Haochen, Jeff Z. and Ma, Tengyu and Liang, Percy},
  booktitle = 	 {Proceedings of the 39th International Conference on Machine Learning},
  pages = 	 {19847--19878},
  year = 	 {2022},
  editor = 	 {Chaudhuri, Kamalika and Jegelka, Stefanie and Song, Le and Szepesvari, Csaba and Niu, Gang and Sabato, Sivan},
  volume = 	 {162},
  series = 	 {Proceedings of Machine Learning Research},
  month = 	 {17--23 Jul},
  publisher =    {PMLR},
  url = 	 {https://proceedings.mlr.press/v162/shen22d.html},
}

@inproceedings{NEURIPS2021_27debb43,
 author = {HaoChen, Jeff Z. and Wei, Colin and Gaidon, Adrien and Ma, Tengyu},
 booktitle = {Advances in Neural Information Processing Systems},
 editor = {M. Ranzato and A. Beygelzimer and Y. Dauphin and P.S. Liang and J. Wortman Vaughan},
 pages = {5000--5011},
 publisher = {Curran Associates, Inc.},
 title = {Provable Guarantees for Self-Supervised Deep Learning with Spectral Contrastive Loss},
 url = {https://proceedings.neurips.cc/paper_files/paper/2021/file/27debb435021eb68b3965290b5e24c49-Paper.pdf},
 volume = {34},
 year = {2021}
}

@inproceedings{NEURIPS2021_2dace78f,
 author = {Nozawa, Kento and Sato, Issei},
 booktitle = {Advances in Neural Information Processing Systems},
 editor = {M. Ranzato and A. Beygelzimer and Y. Dauphin and P.S. Liang and J. Wortman Vaughan},
 pages = {5784--5797},
 publisher = {Curran Associates, Inc.},
 title = {Understanding Negative Samples in Instance Discriminative Self-supervised Representation Learning},
 url = {https://proceedings.neurips.cc/paper_files/paper/2021/file/2dace78f80bc92e6d7493423d729448e-Paper.pdf},
 volume = {34},
 year = {2021}
}

@INPROCEEDINGS{5206848,
  author={Deng, Jia and Dong, Wei and Socher, Richard and Li, Li-Jia and Kai Li and Li Fei-Fei},
  booktitle={2009 IEEE Conference on Computer Vision and Pattern Recognition}, 
  title={ImageNet: A large-scale hierarchical image database}, 
  year={2009},
  volume={},
  number={},
  pages={248-255},
  doi={10.1109/CVPR.2009.5206848}}

@inproceedings{NIPS2016_90e13578,
 author = {Vinyals, Oriol and Blundell, Charles and Lillicrap, Timothy and kavukcuoglu, koray and Wierstra, Daan},
 booktitle = {Advances in Neural Information Processing Systems},
 editor = {D. Lee and M. Sugiyama and U. Luxburg and I. Guyon and R. Garnett},
 pages = {},
 publisher = {Curran Associates, Inc.},
 title = {Matching Networks for One Shot Learning},
 url = {https://proceedings.neurips.cc/paper_files/paper/2016/file/90e1357833654983612fb05e3ec9148c-Paper.pdf},
 volume = {29},
 year = {2016}
}

@InProceedings{He_2016_CVPR,
author = {He, Kaiming and Zhang, Xiangyu and Ren, Shaoqing and Sun, Jian},
title = {Deep Residual Learning for Image Recognition},
booktitle = {Proceedings of the IEEE Conference on Computer Vision and Pattern Recognition (CVPR)},
month = {June},
year = {2016}
}

@inproceedings{NEURIPS2019_ddf35421,
 author = {Bachman, Philip and Hjelm, R Devon and Buchwalter, William},
 booktitle = {Advances in Neural Information Processing Systems},
 editor = {H. Wallach and H. Larochelle and A. Beygelzimer and F. d\textquotesingle Alch\'{e}-Buc and E. Fox and R. Garnett},
 pages = {},
 publisher = {Curran Associates, Inc.},
 title = {Learning Representations by Maximizing Mutual Information Across Views},
 url = {https://proceedings.neurips.cc/paper_files/paper/2019/file/ddf354219aac374f1d40b7e760ee5bb7-Paper.pdf},
 volume = {32},
 year = {2019}
}

@inproceedings{NEURIPS2020_4c2e5eaa,
 author = {Tian, Yonglong and Sun, Chen and Poole, Ben and Krishnan, Dilip and Schmid, Cordelia and Isola, Phillip},
 booktitle = {Advances in Neural Information Processing Systems},
 editor = {H. Larochelle and M. Ranzato and R. Hadsell and M.F. Balcan and H. Lin},
 pages = {6827--6839},
 publisher = {Curran Associates, Inc.},
 title = {What Makes for Good Views for Contrastive Learning?},
 url = {https://proceedings.neurips.cc/paper_files/paper/2020/file/4c2e5eaae9152079b9e95845750bb9ab-Paper.pdf},
 volume = {33},
 year = {2020}
}

@misc{arora2019theoreticalanalysiscontrastiveunsupervised,
      title={A Theoretical Analysis of Contrastive Unsupervised Representation Learning}, 
      author={Sanjeev Arora and Hrishikesh Khandeparkar and Mikhail Khodak and Orestis Plevrakis and Nikunj Saunshi},
      year={2019},
      eprint={1902.09229},
      archivePrefix={arXiv},
      primaryClass={cs.LG},
      url={https://arxiv.org/abs/1902.09229}, 
}

@inproceedings{
wang2022chaos,
title={Chaos is a Ladder: A New Theoretical Understanding of Contrastive Learning via Augmentation Overlap},
author={Yifei Wang and Qi Zhang and Yisen Wang and Jiansheng Yang and Zhouchen Lin},
booktitle={International Conference on Learning Representations},
year={2022},
url={https://openreview.net/forum?id=ECvgmYVyeUz}
}

@InProceedings{pmlr-v139-radford21a,
  title = 	 {Learning Transferable Visual Models From Natural Language Supervision},
  author =       {Radford, Alec and Kim, Jong Wook and Hallacy, Chris and Ramesh, Aditya and Goh, Gabriel and Agarwal, Sandhini and Sastry, Girish and Askell, Amanda and Mishkin, Pamela and Clark, Jack and Krueger, Gretchen and Sutskever, Ilya},
  booktitle = 	 {Proceedings of the 38th International Conference on Machine Learning},
  pages = 	 {8748--8763},
  year = 	 {2021},
  editor = 	 {Meila, Marina and Zhang, Tong},
  volume = 	 {139},
  series = 	 {Proceedings of Machine Learning Research},
  month = 	 {18--24 Jul},
  publisher =    {PMLR},
  url = 	 {https://proceedings.mlr.press/v139/radford21a.html},
}

@InProceedings{pmlr-v195-parulekar23a,
  title = 	 {InfoNCE Loss Provably Learns Cluster-Preserving Representations},
  author =       {Parulekar, Advait and Collins, Liam and Shanmugam, Karthikeyan and Mokhtari, Aryan and Shakkottai, Sanjay},
  booktitle = 	 {Proceedings of Thirty Sixth Conference on Learning Theory},
  pages = 	 {1914--1961},
  year = 	 {2023},
  editor = 	 {Neu, Gergely and Rosasco, Lorenzo},
  volume = 	 {195},
  series = 	 {Proceedings of Machine Learning Research},
  month = 	 {12--15 Jul},
  publisher =    {PMLR},
  url = 	 {https://proceedings.mlr.press/v195/parulekar23a.html},
}

@inproceedings{zimmermann2021contrastive,
  title={Contrastive learning inverts the data generating process},
  author={Zimmermann, Roland S and Sharma, Yash and Schneider, Steffen and Bethge, Matthias and Brendel, Wieland},
  booktitle={International Conference on Machine Learning},
  pages={12979--12990},
  year={2021},
  organization={PMLR}
}

@inproceedings{wang2021understanding,
  title={Understanding the behaviour of contrastive loss},
  author={Wang, Feng and Liu, Huaping},
  booktitle={Proceedings of the IEEE/CVF conference on computer vision and pattern recognition},
  pages={2495--2504},
  year={2021}
}

@inproceedings{wang2020understanding,
  title={Understanding contrastive representation learning through alignment and uniformity on the hypersphere},
  author={Wang, Tongzhou and Isola, Phillip},
  booktitle={International Conference on Machine Learning},
  pages={9929--9939},
  year={2020},
  organization={PMLR}
}

@inproceedings{tosh2021contrastive,
  title={Contrastive learning, multi-view redundancy, and linear models},
  author={Tosh, Christopher and Krishnamurthy, Akshay and Hsu, Daniel},
  booktitle={Algorithmic Learning Theory},
  pages={1179--1206},
  year={2021},
  organization={PMLR}
}

@InProceedings{pmlr-v162-saunshi22a,
  title = 	 {Understanding Contrastive Learning Requires Incorporating Inductive Biases},
  author =       {Saunshi, Nikunj and Ash, Jordan and Goel, Surbhi and Misra, Dipendra and Zhang, Cyril and Arora, Sanjeev and Kakade, Sham and Krishnamurthy, Akshay},
  booktitle = 	 {Proceedings of the 39th International Conference on Machine Learning},
  pages = 	 {19250--19286},
  year = 	 {2022},
  editor = 	 {Chaudhuri, Kamalika and Jegelka, Stefanie and Song, Le and Szepesvari, Csaba and Niu, Gang and Sabato, Sivan},
  volume = 	 {162},
  series = 	 {Proceedings of Machine Learning Research},
  month = 	 {17--23 Jul},
  publisher =    {PMLR},
  url = 	 {https://proceedings.mlr.press/v162/saunshi22a.html},
}

@article{chen2021intriguing,
  title={Intriguing properties of contrastive losses},
  author={Chen, Ting and Luo, Calvin and Li, Lala},
  journal={Advances in Neural Information Processing Systems},
  volume={34},
  pages={11834--11845},
  year={2021}
}

@InProceedings{pmlr-v151-ash22a,
  title = 	 {Investigating the Role of Negatives in Contrastive Representation Learning},
  author =       {Ash, Jordan and Goel, Surbhi and Krishnamurthy, Akshay and Misra, Dipendra},
  booktitle = 	 {Proceedings of The 25th International Conference on Artificial Intelligence and Statistics},
  pages = 	 {7187--7209},
  year = 	 {2022},
  editor = 	 {Camps-Valls, Gustau and Ruiz, Francisco J. R. and Valera, Isabel},
  volume = 	 {151},
  series = 	 {Proceedings of Machine Learning Research},
  month = 	 {28--30 Mar},
  publisher =    {PMLR},
  url = 	 {https://proceedings.mlr.press/v151/ash22a.html},
}

@article{
li2024understanding,
title={Understanding and Improving Transfer Learning of Deep Models via Neural Collapse},
author={Xiao Li and Sheng Liu and Jinxin Zhou and Xinyu Lu and Carlos Fernandez-Granda and Zhihui Zhu and Qing Qu},
journal={Transactions on Machine Learning Research},
issn={2835-8856},
year={2024},
url={https://openreview.net/forum?id=o8r84MzTQB},
note={}
}

@InProceedings{pmlr-v119-goldblum20a,
  title = 	 {Unraveling Meta-Learning: Understanding Feature Representations for Few-Shot Tasks},
  author =       {Goldblum, Micah and Reich, Steven and Fowl, Liam and Ni, Renkun and Cherepanova, Valeriia and Goldstein, Tom},
  booktitle = 	 {Proceedings of the 37th International Conference on Machine Learning},
  pages = 	 {3607--3616},
  year = 	 {2020},
  editor = 	 {III, Hal Daumé and Singh, Aarti},
  volume = 	 {119},
  series = 	 {Proceedings of Machine Learning Research},
  month = 	 {13--18 Jul},
  publisher =    {PMLR},
  url = 	 {https://proceedings.mlr.press/v119/goldblum20a.html},
}

@INPROCEEDINGS{10377311,
  author={Wang, Zijian and Luo, Yadan and Zheng, Liang and Huang, Zi and Baktashmotlagh, Mahsa},
  booktitle={2023 IEEE/CVF International Conference on Computer Vision (ICCV)}, 
  title={How Far Pre-trained Models Are from Neural Collapse on the Target Dataset Informs their Transferability}, 
  year={2023},
  volume={},
  number={},
  pages={5526-5535},
  doi={10.1109/ICCV51070.2023.00511}}

@misc{balestriero2023cookbookselfsupervisedlearning,
      title={A Cookbook of Self-Supervised Learning}, 
      author={Randall Balestriero and Mark Ibrahim and Vlad Sobal and Ari Morcos and Shashank Shekhar and Tom Goldstein and Florian Bordes and Adrien Bardes and Gregoire Mialon and Yuandong Tian and Avi Schwarzschild and Andrew Gordon Wilson and Jonas Geiping and Quentin Garrido and Pierre Fernandez and Amir Bar and Hamed Pirsiavash and Yann LeCun and Micah Goldblum},
      year={2023},
      eprint={2304.12210},
      archivePrefix={arXiv},
      primaryClass={cs.LG},
      url={https://arxiv.org/abs/2304.12210}, 
}

@article{
doi:10.1073/pnas.2015509117,
author = {Vardan Papyan  and X. Y. Han  and David L. Donoho },
title = {Prevalence of neural collapse during the terminal phase of deep learning training},
journal = {Proceedings of the National Academy of Sciences},
volume = {117},
number = {40},
pages = {24652-24663},
year = {2020},
doi = {10.1073/pnas.2015509117},
URL = {https://www.pnas.org/doi/abs/10.1073/pnas.2015509117},
eprint = {https://www.pnas.org/doi/pdf/10.1073/pnas.2015509117},
}

@inproceedings{
han2022neural,
title={Neural Collapse Under {MSE} Loss: Proximity to and Dynamics on the Central Path},
author={X.Y. Han and Vardan Papyan and David L. Donoho},
booktitle={International Conference on Learning Representations},
year={2022},
url={https://openreview.net/forum?id=w1UbdvWH_R3}
}

@misc{galanti2023generalizationboundsfewshottransfer,
      title={Generalization Bounds for Few-Shot Transfer Learning with Pretrained Classifiers}, 
      author={Tomer Galanti and András György and Marcus Hutter},
      year={2023},
      eprint={2212.12532},
      archivePrefix={arXiv},
      primaryClass={cs.LG},
      url={https://arxiv.org/abs/2212.12532}, 
}

@inproceedings{
galanti2022improved,
title={Improved Generalization Bounds for Transfer Learning via Neural Collapse},
author={Tomer Galanti and Andr{\'a}s Gy{\"o}rgy and Marcus Hutter},
booktitle={First Workshop on Pre-training: Perspectives, Pitfalls, and Paths Forward at ICML 2022},
year={2022},
url={https://openreview.net/forum?id=VrK7pKwOhT_}
}

@inproceedings{
zhou2022are,
title={Are All Losses Created Equal: A Neural Collapse Perspective},
author={Jinxin Zhou and Chong You and Xiao Li and Kangning Liu and Sheng Liu and Qing Qu and Zhihui Zhu},
booktitle={Advances in Neural Information Processing Systems},
editor={Alice H. Oh and Alekh Agarwal and Danielle Belgrave and Kyunghyun Cho},
year={2022},
url={https://openreview.net/forum?id=8rfYWE3nyXl}
}

@InProceedings{Caron_2021_ICCV,
    author    = {Caron, Mathilde and Touvron, Hugo and Misra, Ishan and J\'egou, Herv\'e and Mairal, Julien and Bojanowski, Piotr and Joulin, Armand},
    title     = {Emerging Properties in Self-Supervised Vision Transformers},
    booktitle = {Proceedings of the IEEE/CVF International Conference on Computer Vision (ICCV)},
    month     = {October},
    year      = {2021},
    pages     = {9650-9660}
}

@inproceedings{
galanti2022on,
title={On the Role of Neural Collapse in Transfer Learning},
author={Tomer Galanti and Andr{\'a}s Gy{\"o}rgy and Marcus Hutter},
booktitle={International Conference on Learning Representations},
year={2022},
url={https://openreview.net/forum?id=SwIp410B6aQ}
}

@inproceedings{
shaul2023reverse,
title={Reverse Engineering Self-Supervised Learning},
author={Ido Ben-Shaul and Ravid Shwartz-Ziv and Tomer Galanti and Shai Dekel and Yann LeCun},
booktitle={Thirty-seventh Conference on Neural Information Processing Systems},
year={2023},
url={https://openreview.net/forum?id=NsVEjx6YPd}
}

@inproceedings{
weng2025clusteringpropertiesselfsupervisedlearning,
title={Clustering Properties of Self-Supervised Learning},
author={Xi Weng and Jianing An and Xudong Ma and Binhang Qi and Jie Luo and Xi Yang and Jin Song Dong and Lei Huang},
booktitle={Forty-second International Conference on Machine Learning},
year={2025},
url={https://openreview.net/forum?id=j7tKsPQotr}
}

@article{zhou2021ibot,
  title={iBOT: Image BERT Pre-Training with Online Tokenizer},
  author={Zhou, Jinghao and Wei, Chen and Wang, Huiyu and Shen, Wei and Xie, Cihang and Yuille, Alan and Kong, Tao},
  journal={International Conference on Learning Representations (ICLR)},
  year={2022}
}

@misc{luthra2025alignmentsupervisedselfsupervisedcontrastive,
      title={On the Alignment Between Supervised and Self-Supervised Contrastive Learning}, 
      author={Achleshwar Luthra and Priyadarsi Mishra and Tomer Galanti},
      year={2025},
      eprint={2510.08852},
      archivePrefix={arXiv},
      primaryClass={cs.LG},
      url={https://arxiv.org/abs/2510.08852}, 
}

@inproceedings{DBLP:conf/cvpr/AssranDMBVRLB23,
  author={Mahmoud Assran and Quentin Duval and Ishan Misra and Piotr Bojanowski and Pascal Vincent and Michael G. Rabbat and Yann LeCun and Nicolas Ballas},
  title={Self-Supervised Learning from Images with a Joint-Embedding Predictive Architecture},
  year={2023},
  cdate={1672531200000},
  pages={15619-15629},
  url={https://doi.org/10.1109/CVPR52729.2023.01499},
  booktitle={CVPR},
}

@misc{oquab2024dinov2learningrobustvisual,
      title={DINOv2: Learning Robust Visual Features without Supervision}, 
      author={Maxime Oquab and Timothée Darcet and Théo Moutakanni and Huy Vo and Marc Szafraniec and Vasil Khalidov and Pierre Fernandez and Daniel Haziza and Francisco Massa and Alaaeldin El-Nouby and Mahmoud Assran and Nicolas Ballas and Wojciech Galuba and Russell Howes and Po-Yao Huang and Shang-Wen Li and Ishan Misra and Michael Rabbat and Vasu Sharma and Gabriel Synnaeve and Hu Xu and Hervé Jegou and Julien Mairal and Patrick Labatut and Armand Joulin and Piotr Bojanowski},
      year={2024},
      eprint={2304.07193},
      archivePrefix={arXiv},
      primaryClass={cs.CV},
      url={https://arxiv.org/abs/2304.07193}, 
}

@inproceedings{zbontar2021barlow,
  title={Barlow twins: Self-supervised learning via redundancy reduction},
  author={Zbontar, Jure and Jing, Li and Misra, Ishan and LeCun, Yann and Deny, St{\'e}phane},
  booktitle={International conference on machine learning},
  pages={12310--12320},
  year={2021},
  organization={PMLR}
}

@inproceedings{he2022masked,
  title={Masked autoencoders are scalable vision learners},
  author={He, Kaiming and Chen, Xinlei and Xie, Saining and Li, Yanghao and Doll{\'a}r, Piotr and Girshick, Ross},
  booktitle={Proceedings of the IEEE/CVF conference on computer vision and pattern recognition},
  pages={16000--16009},
  year={2022}
}

@InProceedings{pmlr-v108-mcallester20a,
  title = 	 {Formal Limitations on the Measurement of Mutual Information},
  author =       {McAllester, David and Stratos, Karl},
  booktitle = 	 {Proceedings of the Twenty Third International Conference on Artificial Intelligence and Statistics},
  pages = 	 {875--884},
  year = 	 {2020},
  editor = 	 {Chiappa, Silvia and Calandra, Roberto},
  volume = 	 {108},
  series = 	 {Proceedings of Machine Learning Research},
  month = 	 {26--28 Aug},
  publisher =    {PMLR},
  url = 	 {https://proceedings.mlr.press/v108/mcallester20a.html},
}

@inproceedings{
balestriero2024the,
title={The Birth of Self Supervised Learning: A Supervised Theory},
author={Randall Balestriero and Yann LeCun},
booktitle={NeurIPS 2024 Workshop: Self-Supervised Learning - Theory and Practice},
year={2024},
url={https://openreview.net/forum?id=NhYAjAAdQT}
}

@inproceedings{10.1007/978-3-031-19809-0_38,
author = {Yeh, Chun-Hsiao and Hong, Cheng-Yao and Hsu, Yen-Chi and Liu, Tyng-Luh and Chen, Yubei and LeCun, Yann},
title = {Decoupled Contrastive Learning},
year = {2022},
isbn = {978-3-031-19808-3},
publisher = {Springer-Verlag},
address = {Berlin, Heidelberg},
url = {https://doi.org/10.1007/978-3-031-19809-0_38},
doi = {10.1007/978-3-031-19809-0_38},
booktitle = {Computer Vision – ECCV 2022: 17th European Conference, Tel Aviv, Israel, October 23–27, 2022, Proceedings, Part XXVI},
pages = {668–684},
numpages = {17},
location = {Tel Aviv, Israel}
}

@inproceedings{
Tschannen2020On,
title={On Mutual Information Maximization for Representation Learning},
author={Michael Tschannen and Josip Djolonga and Paul K. Rubenstein and Sylvain Gelly and Mario Lucic},
booktitle={International Conference on Learning Representations},
year={2020},
url={https://openreview.net/forum?id=rkxoh24FPH}
}

@inproceedings{
xue2024investigating,
title={Investigating the Benefits of Projection Head for Representation Learning},
author={Yihao Xue and Eric Gan and Jiayi Ni and Siddharth Joshi and Baharan Mirzasoleiman},
booktitle={The Twelfth International Conference on Learning Representations},
year={2024},
url={https://openreview.net/forum?id=GgEAdqYPNA}
}

@inproceedings{
ouyang2025projection,
title={Projection Head is Secretly an Information Bottleneck},
author={Zhuo Ouyang and Kaiwen Hu and Qi Zhang and Yifei Wang and Yisen Wang},
booktitle={The Thirteenth International Conference on Learning Representations},
year={2025},
url={https://openreview.net/forum?id=L0evcuybH5}
}

@misc{gui2023unravelingprojectionheadscontrastive,
      title={Unraveling Projection Heads in Contrastive Learning: Insights from Expansion and Shrinkage}, 
      author={Yu Gui and Cong Ma and Yiqiao Zhong},
      year={2023},
      eprint={2306.03335},
      archivePrefix={arXiv},
      primaryClass={stat.ML},
      url={https://arxiv.org/abs/2306.03335}, 
}

@inproceedings{
balestriero2022contrastive,
title={Contrastive and Non-Contrastive Self-Supervised Learning Recover Global and Local Spectral Embedding Methods},
author={Randall Balestriero and Yann LeCun},
booktitle={Advances in Neural Information Processing Systems},
editor={Alice H. Oh and Alekh Agarwal and Danielle Belgrave and Kyunghyun Cho},
year={2022},
url={https://openreview.net/forum?id=jQgsZDspz5h}
}

@inproceedings{
wei2021theoretical,
title={Theoretical Analysis of Self-Training with Deep Networks on Unlabeled Data},
author={Colin Wei and Kendrick Shen and Yining Chen and Tengyu Ma},
booktitle={International Conference on Learning Representations},
year={2021},
url={https://openreview.net/forum?id=rC8sJ4i6kaH}
}

@inproceedings{
alon2024optimal,
title={Optimal Sample Complexity of Contrastive Learning},
author={Noga Alon and Dmitrii Avdiukhin and Dor Elboim and Orr Fischer and Grigory Yaroslavtsev},
booktitle={The Twelfth International Conference on Learning Representations},
year={2024},
url={https://openreview.net/forum?id=NU9AYHJvYe}
}

@misc{gupta2022understandingimprovingroleprojection,
      title={Understanding and Improving the Role of Projection Head in Self-Supervised Learning}, 
      author={Kartik Gupta and Thalaiyasingam Ajanthan and Anton van den Hengel and Stephen Gould},
      year={2022},
      eprint={2212.11491},
      archivePrefix={arXiv},
      primaryClass={cs.LG},
      url={https://arxiv.org/abs/2212.11491}, 
}

@inproceedings{
shwartzziv2023an,
title={An Information Theory Perspective on Variance-Invariance-Covariance Regularization},
author={Ravid Shwartz-Ziv and Randall Balestriero and Kenji Kawaguchi and Tim G. J. Rudner and Yann LeCun},
booktitle={Thirty-seventh Conference on Neural Information Processing Systems},
year={2023},
url={https://openreview.net/forum?id=KipjqOPaZ0}
}

@article{you2017large,
  title={Large batch training of convolutional networks},
  author={You, Yang and Gitman, Igor and Ginsburg, Boris},
  journal={arXiv preprint arXiv:1708.03888},
  year={2017}
}

@article{loshchilov2016sgdr,
  title={Sgdr: Stochastic gradient descent with warm restarts},
  author={Loshchilov, Ilya and Hutter, Frank},
  journal={arXiv preprint arXiv:1608.03983},
  year={2016}
}

@article{goyal2017accurate,
  title={Accurate, large minibatch sgd: Training imagenet in 1 hour},
  author={Goyal, Priya and Doll{\'a}r, Piotr and Girshick, Ross and Noordhuis, Pieter and Wesolowski, Lukasz and Kyrola, Aapo and Tulloch, Andrew and Jia, Yangqing and He, Kaiming},
  journal={arXiv preprint arXiv:1706.02677},
  year={2017}
}

@misc{luthra2025selfsupervisedcontrastivelearningapproximately,
      title={Self-Supervised Contrastive Learning is Approximately Supervised Contrastive Learning}, 
      author={Achleshwar Luthra and Tianbao Yang and Tomer Galanti},
      year={2025},
      eprint={2506.04411},
      archivePrefix={arXiv},
      primaryClass={cs.LG},
      url={https://arxiv.org/abs/2506.04411}, 
}

@InProceedings{Zhai_2023_ICCV,
    author    = {Zhai, Xiaohua and Mustafa, Basil and Kolesnikov, Alexander and Beyer, Lucas},
    title     = {Sigmoid Loss for Language Image Pre-Training},
    booktitle = {Proceedings of the IEEE/CVF International Conference on Computer Vision (ICCV)},
    month     = {October},
    year      = {2023},
    pages     = {11975-11986}
}

@misc{dosovitskiy2021imageworth16x16words,
      title={An Image is Worth 16x16 Words: Transformers for Image Recognition at Scale}, 
      author={Alexey Dosovitskiy and Lucas Beyer and Alexander Kolesnikov and Dirk Weissenborn and Xiaohua Zhai and Thomas Unterthiner and Mostafa Dehghani and Matthias Minderer and Georg Heigold and Sylvain Gelly and Jakob Uszkoreit and Neil Houlsby},
      year={2021},
      eprint={2010.11929},
      archivePrefix={arXiv},
      primaryClass={cs.CV},
      url={https://arxiv.org/abs/2010.11929}, 
}

@inproceedings{
loshchilov2018decoupled,
title={Decoupled Weight Decay Regularization},
author={Ilya Loshchilov and Frank Hutter},
booktitle={International Conference on Learning Representations},
year={2019},
url={https://openreview.net/forum?id=Bkg6RiCqY7},
}

@misc{chen2023palijointlyscaledmultilinguallanguageimage,
      title={PaLI: A Jointly-Scaled Multilingual Language-Image Model}, 
      author={Xi Chen and Xiao Wang and Soravit Changpinyo and AJ Piergiovanni and Piotr Padlewski and Daniel Salz and Sebastian Goodman and Adam Grycner and Basil Mustafa and Lucas Beyer and Alexander Kolesnikov and Joan Puigcerver and Nan Ding and Keran Rong and Hassan Akbari and Gaurav Mishra and Linting Xue and Ashish Thapliyal and James Bradbury and Weicheng Kuo and Mojtaba Seyedhosseini and Chao Jia and Burcu Karagol Ayan and Carlos Riquelme and Andreas Steiner and Anelia Angelova and Xiaohua Zhai and Neil Houlsby and Radu Soricut},
      year={2023},
      eprint={2209.06794},
      archivePrefix={arXiv},
      primaryClass={cs.CV},
      url={https://arxiv.org/abs/2209.06794}, 
}
\bibliographystyle{icml2026}

\newpage
\onecolumn
\appendix

\section{Additional Experimental Details}\label{app:details}

In this section we describe the experimental settings. Please refer to App.~\ref{app:details} for the complete details.

{\bf Datasets and augmentations. \enspace} We conduct all our experiments on the standard image classification dataset -- mini-ImageNet~\citep{NIPS2016_90e13578} which is a subset of ImageNet-1K~\cite{5206848}. It contains images at their original resolution (\(224 \times 224\)) and has 50000 train, 10000 validation, and 5000 test images, with a total 100 of 1000 ImageNet classes. 

{\bf Methods and optimizers. \enspace} We consider several families of self-supervised learning methods, including contrastive learning (SimCLR, VICReg), masked modeling (I-JEPA, MAE), distillation-based approaches (DINO-v2), and multimodal learning (CLIP, SigLIP). For experiments focused on analyzing learning dynamics, we train SimCLR, VICReg, MAE, and DINO-v2 models from scratch on mini-ImageNet. 

For SimCLR, we use DCL loss~\citep{10.1007/978-3-031-19809-0_38} in place of the standard InfoNCE loss and optimize the model using LARS~\citep{you2017large}. The momentum is set to $0.9$ and the weight decay to $\expnum{1}{6}$. The learning rate is scaled linearly with the batch size as $0.3 \cdot \lfloor B / 256 \rfloor$~\citep{pmlr-v119-chen20j}., and we use a batch size of $1024$ throughout training. We apply a linear warm-up~\citep{goyal2017accurate} for the first $10$ epochs, followed by a cosine learning rate schedule without restarts~\citep{loshchilov2016sgdr} for the remainder of training.

For VicReg, we use a ResNet-50 backbone \citep{He_2016_CVPR} with a 2-layer projection head (hidden and output dimensions 2048). The VICReg loss \citep{bardes2022vicreg} combines three objectives: invariance loss for view similarity, variance loss to prevent collapse, and covariance loss for dimension decorrelation, with coefficients $\lambda = 25.0$, $\mu = 25.0$, and $\nu = 1.0$. The model is optimized using AdamW \citep{loshchilov2018decoupled} with weight decay 0.05. The learning rate is scaled as $0.0005 \cdot (B \times W)/256$ where $B$ is the batch size per GPU, with 10-epoch linear warm-up followed by cosine decay \citep{loshchilov2016sgdr} to 0.0. 

For MAE~\citep{he2022masked}, we use mean-squared error between reconstructed and input image patches as our objective function. The target pixel values are normalized per-patch as it encourages the model to learn richer, contrast-invariant semantic features. We use a masking ratio of $0.75$, meaning that $75\%$ of image patches are masked and only the remaining patches are passed to the encoder. The decoder is a lightweight Vision Transformer with $8$ transformer blocks, $16$ attention heads, and an embedding dimension of $512$. We follow the same learning rate scheduler and scaler as SimCLR, except the base learning rate is set to $\expnum{1.5}{4}$ (instead of $0.3$). Additionally, weight-decay for all layers is set to $0.05$ except LayerNorm and bias terms. We use AdamW~\citep{loshchilov2018decoupled} optimizer following standard MAE setup.

For DINO, we use the DINOv2 framework with a Vision Transformer Base (ViT-B/16) architecture combining three loss objectives: DINO loss \citep{Caron_2021_ICCV} for global consistency, iBOT patch loss \citep{zhou2021ibot} for local feature learning with masked patches, and KoLeo loss for feature uniformity. The model is optimized using AdamW \citep{loshchilov2018decoupled} with weight decay 0.04. The learning rate is scaled as $0.001 \cdot \lfloor (B \times W)/256 \rfloor$ (batch size 256), with 10-epoch linear warm-up followed by cosine decay \citep{loshchilov2016sgdr} to $1 \times 10^{-5}$. The teacher network uses EMA with cosine momentum schedule (0.992 to 1.0) and temperature warm-up (0.04 to 0.07 over 30\% of training). Data augmentation employs 2 global and 8 local crops with block masking (ratio 0.6) on global crops for iBOT.

{\bf SimCLR Augmentations.\enspace} As described in SimCLR~\citep{pmlr-v119-chen20j}, we use the following pipeline: random resized cropping to \(224 \times 224\), random horizontal flipping, color jittering (brightness, contrast, saturation: $0.8$; hue: $0.2$), random grayscale conversion (\(p=0.2\)), and Gaussian blur (applied with probability $0.1$ using a \(3 \times 3\) kernel and \(\sigma = 1.5\)).

{\bf VicReg Augmentations.\enspace} Following VICReg~\cite{bardes2022vicreg}, we adopt the SimCLR-style image augmentation protocol and symmetrized across views. Two random crops are sampled from each image and resized to $224 \times 224$. Each view undergoes random horizontal flipping, color jittering (brightness: $0.4$, contrast: $0.4$, saturation: $0.2$, hue: $0.1$; applied with probability $0.8$), random grayscale conversion ($p = 0.2$), Gaussian blur ($p = 0.5$, kernel size $23$), and solarization ($p = 0.1$). Finally, images are normalized per channel using the ImageNet statistics.

{\bf MAE Augmentations. \enspace} For MAE, we do not apply strong appearance-based augmentations such as color jittering, grayscale conversion, or Gaussian blurring, as these can adversely affect the pixel reconstruction objective. Instead, we use a RandomResizedCrop with {\em bicubic} interpolation and a scale range of $(0.2, 1.0)$, followed by normalization using ImageNet-1K statistics.

{\bf DINO-v2 Augmentations. \enspace} Following DINOv2~\cite{oquab2024dinov2learningrobustvisual}, we use a multi-crop image augmentation strategy inherited from DINO. For each image, two global views are sampled using random resized crops and resized to $224 \times 224$, along with several local views resized to $96 \times 96$. Each view is independently augmented using random horizontal flipping, color jittering of brightness, contrast, saturation, and hue (applied with probability $0.8$), and random grayscale conversion ($p = 0.2$). Gaussian blur is applied to the global views with different probabilities (strong blur on one global crop and weaker blur on the other), and solarization is applied to one of the global crops with probability $0.2$.\\

{\bf Backbone architectures. \enspace} We use multiple backbone architectures to demonstrate that our observations are not specific to a particular model family. Specifically, we consider ResNet-18 and ResNet-50~\citep{He_2016_CVPR} encoders with a width multiplier of $2$, as well as a Vision Transformer (ViT-Base)~\citep{dosovitskiy2021imageworth16x16words}. 

The ViT-Base (ViT-B/16) architecture consists of $12$ transformer layers, each with $12$ attention heads and a hidden dimension of $768$. For input images of size $224 \times 224$, we use a patch size of $16 \times 16$, yielding $196$ patch tokens along with a single [CLS] token. The MLP blocks have a hidden dimension of $3072$, and layer normalization is applied prior to each attention and MLP block.

For SimCLR and VicReg, the backbone encoders are followed by a projection head with a standard two-layer MLP architecture composed of: \texttt{Linear($2048 \rightarrow 2048$) $\rightarrow$ ReLU $\rightarrow$ Linear($2048 \rightarrow 128$)}.

\section{Additional Experiments}

{\bf Hierarchical clustering.\enspace} We analyze the CDNV for self-supervised learning models using a hierarchical class structure. We generated 10 semantic superclasses for Mini-ImageNet's 100 classes using GPT 4o to group related categories, such as \textit{birds} (house finch, robin, toucan), \textit{dogs} (golden retriever, boxer, dalmatian), and \textit{marine life} (jellyfish, king crab, coral reef). This allows us to analyze whether SSL models learn representations that respect semantic groupings and how CDNV evolves during training both within and across superclasses.

\begin{figure}[t]
  \centering
  \begin{tabular}{cc}
    \includegraphics[width=0.45\linewidth]{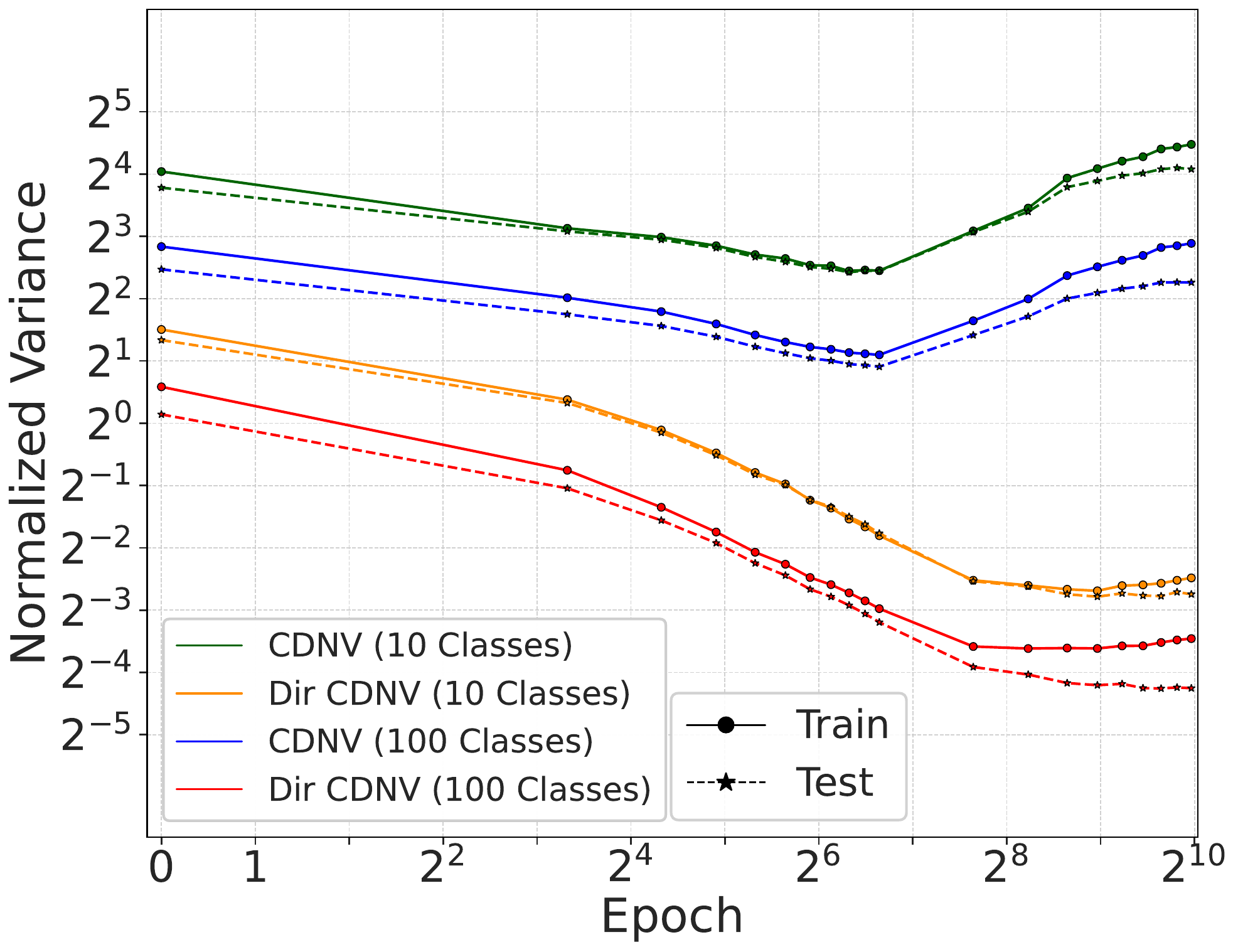} &
    \includegraphics[width=0.45\linewidth]{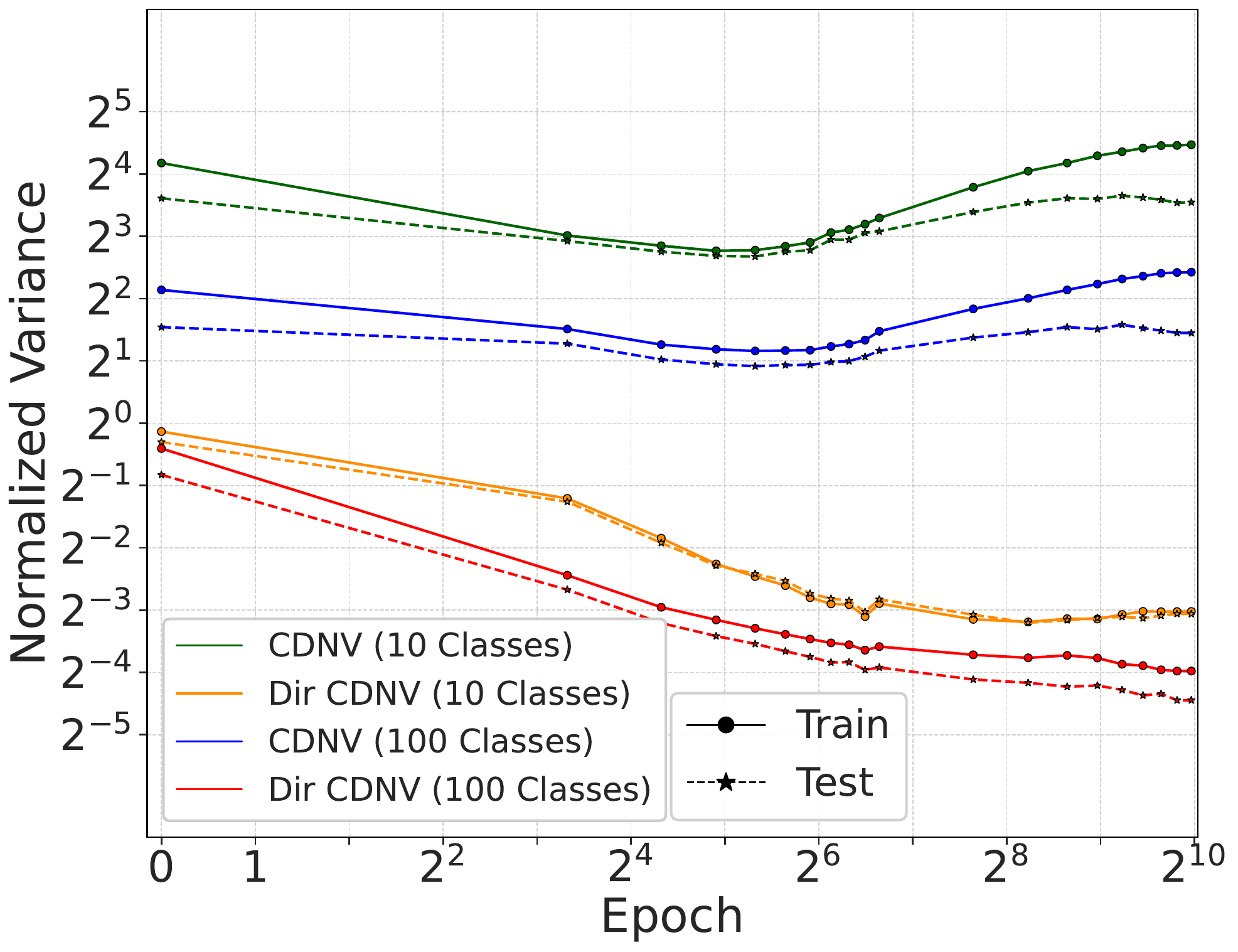} \\
    {\small \textbf{DINO-v2}} & {\small \textbf{VICReg}} \\
  \end{tabular}
  \caption{\textbf{Directional collapse manifests across semantic hierarchies.}
  We compare CDNV and directional CDNV at two levels of granularity: 10 semantic superclasses and 100 fine-grained classes, on both training and test data.
  Directional CDNV decreases much more than CDNV at both granularities, demonstrating that SSL models primarily tighten class geometry along separating directions.}
\end{figure}

\section{Proofs}

Assume all samples are i.i.d.\ within each class and independent across classes.
Let $z_c=f(x_c)$ for $x_c\sim D_c$ have mean $\mu_c:=\EE[z_c]$ and covariance $\Sigma_c:=\Cov(z_c)$ for $c\in\{i,j\}$.
For $m\ge 1$ samples per class, draw support inputs $\{x_{c,s}\}_{s=1}^m$ i.i.d.\ from $D_c$ for $c\in\{i,j\}$,
set $z_{c,s}:=f(x_{c,s})$, and define $\widehat\mu_c:=\frac1m\sum_{s=1}^m z_{c,s}$.
Define the sample-mean errors $\delta_c:=\widehat\mu_c-\mu_c$ for $c\in\{i,j\}$.
Draw a test input $x_i\sim D_i$ independently of the support inputs and set $z_i:=f(x_i)$.
Fix $i\neq j$ with $d_{ij}:=\|\mu_j-\mu_i\|_2>0$ and define $u_{ij}:=\frac{\mu_j-\mu_i}{d_{ij}}$.

\begin{proposition}[Pairwise NCC error with tunable coefficients]\label{thm:ncc-generic-detailed}
Assume $M_{4,c}:=\EE\|z_c-\mu_c\|^4<\infty$ for $c\in\{i,j\}$, and define
\[
v_c:=\tr(\Sigma_c),\qquad
\tilde V_{ij}:=\frac{u_{ij}^\top\Sigma_i u_{ij}}{d_{ij}^2},\qquad
V_{ij}:=\frac{v_i+v_j}{d_{ij}^2},\qquad
\Theta_{ij}:=\frac{M_{4,i}+M_{4,j}}{d_{ij}^4}.
\]
Let $\Delta:=\|z_i-\widehat\mu_j\|_2^2-\|z_i-\widehat\mu_i\|_2^2$ and suppose
$\EE[\Delta]=d_{ij}^2+\tfrac{v_j-v_i}{m}>0$.
For any weights $\lambda_T,\lambda_S,\lambda_Q>0$, write $\kappa:=\lambda_T+\lambda_S+\lambda_Q$ and
\[
a_T:=\frac{4\kappa}{m\,\lambda_T},\qquad
a_S:=\frac{\kappa}{m\,\lambda_S},\qquad
a_Q:=\frac{\kappa}{m^{3}\,\lambda_Q}.
\]
Then
\begin{equation}\label{eq:ncc-generic-detailed}
\Pr(\mathrm{NCC\ predicts\ }j\mid i)
~=~
\Pr(\Delta\le 0)
~\le~
\frac{
\displaystyle
4\,\tilde V_{ij}
+\;a_T\,V_{ij}^{2}
+\;\Big(\tfrac{a_T}{4}+a_S\Big) V_{ij}
+\;a_Q\Big(\Theta_{ij}+2(m-1)V_{ij}^{2}\Big)
}{\Denij}~.
\end{equation}
\end{proposition}

\begin{proof}
Set
\[
A:=z_i-\tfrac{\mu_i+\mu_j}{2},\qquad
X:=(z_i-\mu_i)^\top u_{ij},\qquad
T:=A^\top(\delta_j-\delta_i),
\]
\[
S:=(\mu_j-\mu_i)^\top(\delta_j+\delta_i),\qquad
Q:=\|\delta_j\|_2^2-\|\delta_i\|_2^2.
\]
A direct expansion gives
\[
\Delta=d_{ij}^2-2d_{ij}X-2T+S+Q,
\]
hence $\EE[\Delta]=d_{ij}^2+\tfrac{v_j-v_i}{m}$.

\paragraph{Orthogonality.}
We claim that $\Cov(X,T)=\Cov(X,S)=\Cov(X,Q)=0$.
Indeed, $(\delta_i,\delta_j)$ is independent of $z_i$ and has mean $0$, so
\[
\EE[XT]=\EE \left[X\,A^\top(\delta_j-\delta_i)\right]
=\EE[X\,A^\top]\;\EE[\delta_j-\delta_i]=0,
\]
and since $\EE[T]=0$, this implies $\Cov(X,T)=0$.
Similarly, using independence of $X$ and $(\delta_i,\delta_j)$,
\[
\EE[XS]=\EE[X]\;\EE \left[(\mu_j-\mu_i)^\top(\delta_j+\delta_i)\right]=0,
\]
and since $\EE[S]=0$, we get $\Cov(X,S)=0$.
Finally, $X$ is a function of $z_i$ while $Q=\|\delta_j\|_2^2-\|\delta_i\|_2^2$ is a function of $(\delta_i,\delta_j)$,
and $(\delta_i,\delta_j)$ is independent of $z_i$, so $X$ is independent of $Q$ and thus $\Cov(X,Q)=0$.

Let $U_1:=-2d_{ij}X$, $W_1:=-2T$, $W_2:=S$, $W_3:=Q$ and $V:=W_1+W_2+W_3$.
By bilinearity of covariance and the above orthogonality relations, $\Cov(U_1,V)=0$, hence
\[
\Var(\Delta)=\Var(U_1)+\Var(V),\qquad
\Var(U_1)=4d_{ij}^2\Var(X)=4d_{ij}^2\cdot u_{ij}^\top\Sigma_i u_{ij}
=4d_{ij}^4\tilde V_{ij}.
\]

\paragraph{Weighted variance control for $V$.}
For any $\lambda_T,\lambda_S,\lambda_Q>0$ and $\kappa:=\lambda_T+\lambda_S+\lambda_Q$, the weighted
Cauchy--Schwarz inequality yields
\[
\Var(V)=\Var(W_1+W_2+W_3)
\le \kappa \left(\frac{\Var(W_1)}{\lambda_T}+\frac{\Var(W_2)}{\lambda_S}+\frac{\Var(W_3)}{\lambda_Q}\right).
\]
Since $\Var(W_1)=4\Var(T)$, $\Var(W_2)=\Var(S)$, $\Var(W_3)=\Var(Q)$, we have
\[
\Var(V)\le
\kappa \left(\frac{4\,\Var(T)}{\lambda_T}+\frac{\Var(S)}{\lambda_S}+\frac{\Var(Q)}{\lambda_Q}\right).
\]

\paragraph{Component variances.}
Standard moment calculus and PSD bounds give
\[
\Var(T)=\frac{1}{m}\,\tr \Big((\Sigma_i+\Sigma_j)\big(\Sigma_i+\tfrac14 (\mu_j-\mu_i)(\mu_j-\mu_i)^\top\big)\Big),
\quad
\Var(S)=\frac{1}{m}\,(\mu_j-\mu_i)^\top(\Sigma_i+\Sigma_j)(\mu_j-\mu_i),
\]
\[
\Var(Q)=\sum_{c\in\{i,j\}} \left[\frac{1}{m^3}\big(M_{4,c}-v_c^2\big)+\frac{2(m-1)}{m^3}\tr(\Sigma_c^2)\right].
\]
Using $\tr(AB)\le\sqrt{\tr(A^2)\tr(B^2)}$, $\tr(\Sigma_c^2)\le v_c^2$, and
\[
(\mu_j-\mu_i)^\top(\Sigma_i+\Sigma_j)(\mu_j-\mu_i)\le d_{ij}^2(v_i+v_j),
\]
we obtain
\[
\Var(T)\le \frac{1}{m} \left(v_i^2+v_iv_j+\tfrac14d_{ij}^2(v_i+v_j)\right),\quad
\Var(S)\le \frac{1}{m}d_{ij}^2(v_i+v_j),\quad
\Var(Q)\le \frac{d_{ij}^4}{m^3} \left(\Theta_{ij}+2(m-1)V_{ij}^2\right).
\]
Also $v_i^2+v_iv_j\le (v_i+v_j)^2=d_{ij}^4V_{ij}^2$ and $d_{ij}^2(v_i+v_j)=d_{ij}^4V_{ij}$.

\paragraph{Assemble.}
Combining the pieces,
\[
\Var(\Delta)\le d_{ij}^4 \left[
4\,\tilde V_{ij}
+\frac{4\kappa}{m\,\lambda_T}V_{ij}^2
+\left(\frac{\kappa}{m\,\lambda_T}+\frac{\kappa}{m\,\lambda_S}\right)V_{ij}
+\frac{\kappa}{m^3\,\lambda_Q}\Big(\Theta_{ij}+2(m-1)V_{ij}^2\Big)
\right].
\]
Finally, Chebyshev’s inequality gives
\[
\Pr(\Delta\le 0)\le \frac{\Var(\Delta)}{(\EE\Delta)^2}
=\frac{\Var(\Delta)}{\bigl(d_{ij}^2+\tfrac{v_j-v_i}{m}\bigr)^2},
\]
and abbreviating $a_T,a_S,a_Q$ yields \eqref{eq:ncc-generic-detailed}.
\end{proof}

\begin{restatable}{theorem}{wts}\label{thm:eq-wts}
Let $C'\geq 2$ and $m\geq 1$ be integers. Fix a feature map \(f:\mathcal X\to\mathbb R^{d}\) and class-conditional distributions \(D_{1},\dots,D_{C'}\) over \(\mathcal X\). We have:
\begin{small}
\begin{equation*}
\begin{aligned}
\mathrm{err}^{\mathrm{NCC}}_{m,\mathcal{C}}(f)
~&\le~ \frac{1}{C'}\sum_{i=1}^{C'}\sum_{j\ne i}
\frac{4\,\tilde V_{ij} + \tfrac{12}{m} V^2_{ij} + \tfrac{6}{m} V_{ij}}
{\bigl(1+\tfrac{v_j-v_i}{m\,d_{ij}^{2}}\bigr)^{2}}
\\
&\quad + \frac{1}{C'}\sum_{i=1}^{C'}\sum_{j\ne i}
\frac{\tfrac{3}{m^3}\bigl(\Theta_{ij} + 2(m-1)V^2_{ij}\bigr)}
{\bigl(1+\tfrac{v_j-v_i}{m\,d_{ij}^{2}}\bigr)^{2}}~.
\end{aligned}
\end{equation*}
\end{small}
\end{restatable}

\begin{proof}[Proof of Thm.~\ref{thm:eq-wts}]
Fix $i\neq j$ and apply Prop.~\ref{thm:ncc-generic-detailed} with
$\lambda_T=\lambda_S=\lambda_Q=1$, so $\kappa=3$ and thus
\[
a_T=\frac{4\kappa}{m\lambda_T}=\frac{12}{m},\qquad
a_S=\frac{\kappa}{m\lambda_S}=\frac{3}{m},\qquad
a_Q=\frac{\kappa}{m^3\lambda_Q}=\frac{3}{m^3}.
\]
This gives the stated pairwise bound with numerator
\[
4\tilde V_{ij}+\frac{12}{m}V_{ij}^2+\left(\frac{3}{m}+\frac{3}{m}\right)V_{ij}
+\frac{3}{m^3}\bigl(\Theta_{ij}+2(m-1)V_{ij}^2\bigr).
\]
For multiclass NCC error, use the union bound
$\Pr(\widehat y^{\mathrm{NCC}}(z_i)\neq i)\le\sum_{j\neq i}\Pr(\Delta_{i\to j}\le 0)$,
then average over $i$ to obtain the theorem.
\end{proof}

\optimized*

\begin{proof}[Proof of Thm.~\ref{thm:ncc-full-optimized}]
Fix $i\neq j$ and apply Prop.~\ref{thm:ncc-generic-detailed}.
Group the $\lambda_T$-controlled terms as
\[
a_TV_{ij}^2+\frac{a_T}{4}V_{ij}
=\kappa\cdot \frac{4}{m\lambda_T}\Bigl(V_{ij}^2+\frac14 V_{ij}\Bigr)
=\kappa\cdot \frac{E^1_{ij}}{\lambda_T},
\]
and similarly
\[
a_SV_{ij}=\kappa\cdot \frac{1}{m\lambda_S}V_{ij}=\kappa\cdot \frac{E^2_{ij}}{\lambda_S},
\qquad
a_Q(\Theta_{ij}+2(m-1)V_{ij}^2)=\kappa\cdot \frac{E^3_{ij}}{\lambda_Q},
\]
where
\[
E^1_{ij}:=\frac{4}{m}\Bigl(V_{ij}^2+\tfrac14 V_{ij}\Bigr),\quad
E^2_{ij}:=\frac{1}{m}V_{ij},\quad
E^3_{ij}:=\frac{1}{m^3}\Bigl(\Theta_{ij}+2(m-1)V_{ij}^2\Bigr).
\]
Hence Prop.~\ref{thm:ncc-generic-detailed} yields
\[
\Pr(\Delta_{i\to j}\le 0)\le
\frac{4\tilde V_{ij}+\kappa\Bigl(\frac{E^1_{ij}}{\lambda_T}+\frac{E^2_{ij}}{\lambda_S}+\frac{E^3_{ij}}{\lambda_Q}\Bigr)}
{\Denij}.
\]
By Cauchy--Schwarz in $\mathbb{R}^3$ (take $a_k=\sqrt{\lambda_k}$ and $b_k=\sqrt{E^k_{ij}/\lambda_k}$),
\[
\Bigl(\lambda_T+\lambda_S+\lambda_Q\Bigr)\Bigl(\frac{E^1_{ij}}{\lambda_T}+\frac{E^2_{ij}}{\lambda_S}+\frac{E^3_{ij}}{\lambda_Q}\Bigr)
\ge \Bigl(\sqrt{E^1_{ij}}+\sqrt{E^2_{ij}}+\sqrt{E^3_{ij}}\Bigr)^2,
\]
with equality for $\lambda_T:\lambda_S:\lambda_Q=\sqrt{E^1_{ij}}:\sqrt{E^2_{ij}}:\sqrt{E^3_{ij}}$.
Choosing $\lambda_T=\sqrt{E^1_{ij}}$, $\lambda_S=\sqrt{E^2_{ij}}$, $\lambda_Q=\sqrt{E^3_{ij}}$
gives
\[
\Pr(\Delta_{i\to j}\le 0)\le
\frac{4\tilde V_{ij}+\bigl(\sqrt{E^1_{ij}}+\sqrt{E^2_{ij}}+\sqrt{E^3_{ij}}\bigr)^2}{\Denij}.
\]
Finish by the same union bound over $j\neq i$ and averaging over $i$.
\end{proof}

\orthmulti*

\begin{proof}
Write
$u:=u^{(1)}_{aa'}$, $v:=u^{(2)}_{bb'}$, $d_1:=d^{(1)}_{aa'}$, and $d_2:=d^{(2)}_{bb'}$.
For $c\in[K_1]$ and $t\in[K_2]$, define the joint conditional means
$\mu_{c,t}:=\EE[z\mid y^{(1)}=c,\;y^{(2)}=t]$.
For the fixed pair $(b,b')$ in task $2$, define
$\Gamma_c:=\mu_{c,b}-\mu_{c,b'}$ for each $c\in[K_1]$.

We first bound $u^\top \Gamma_c$ using the directional variance of task $1$.
Fix $c\in[K_1]$. By the law of total variance,
\[
u^\top \Sigma^{(1)}_{c}u
=
\Var(u^\top z\mid y^{(1)}=c)
\ge
\Var \big(\EE[u^\top z\mid y^{(1)}=c,\;y^{(2)}]\mid y^{(1)}=c\big).
\]
Since $y^{(1)}$ and $y^{(2)}$ are independent and $y^{(2)}$ is balanced, conditioned on $y^{(1)}=c$ the variable $y^{(2)}$ is uniform on $[K_2]$. Thus, if we set $m_{c,t}:=u^\top \mu_{c,t}$ and $\bar m_c:=\frac1{K_2}\sum_{t=1}^{K_2} m_{c,t}$, then
\[
\Var \big(\EE[u^\top z\mid y^{(1)}=c,\;y^{(2)}]\mid y^{(1)}=c\big)
=
\frac1{K_2}\sum_{t=1}^{K_2}(m_{c,t}-\bar m_c)^2.
\]
Now, for the two selected classes $b,b'$, we have
\[
|u^\top \Gamma_c|^2
=
|m_{c,b}-m_{c,b'}|^2
\le
2\big((m_{c,b}-\bar m_c)^2+(m_{c,b'}-\bar m_c)^2\big)
\le
2K_2\cdot \frac1{K_2}\sum_{t=1}^{K_2}(m_{c,t}-\bar m_c)^2.
\]
Combining the last three displays gives
\[
|u^\top \Gamma_c|
\le
\sqrt{2K_2\,u^\top \Sigma^{(1)}_c u}
=
d_1\sqrt{2K_2\,\tilde V^{(1)}_{aa',c}}
\le
d_1\sqrt{2K_2\,\tilde V^{(1)}_{aa'}}.
\]

Next we express the task-$2$ pairwise mean gap using the joint means.
Independence and balancedness of $y^{(1)}$ imply that for each $t\in[K_2]$,
$\Pr(y^{(1)}=c\mid y^{(2)}=t)=1/K_1$, so
\[
\mu^{(2)}_{t}
=
\EE[z\mid y^{(2)}=t]
=
\frac1{K_1}\sum_{c=1}^{K_1}\mu_{c,t}.
\]
Therefore,
\[
\mu^{(2)}_{b}-\mu^{(2)}_{b'}
=
\frac1{K_1}\sum_{c=1}^{K_1}(\mu_{c,b}-\mu_{c,b'})
=
\frac1{K_1}\sum_{c=1}^{K_1}\Gamma_c.
\]
Taking inner products with $u$ and using $\mu^{(2)}_{b}-\mu^{(2)}_{b'}=d_2 v$, we get
\[
d_2\,u^\top v
=
u^\top(\mu^{(2)}_{b}-\mu^{(2)}_{b'})
=
\frac1{K_1}\sum_{c=1}^{K_1}u^\top \Gamma_c.
\]
Hence, by the triangle inequality and the bound above,
\[
|d_2\,u^\top v|
\le
\frac1{K_1}\sum_{c=1}^{K_1}|u^\top \Gamma_c|
\le
\frac1{K_1}\sum_{c=1}^{K_1} d_1\sqrt{2K_2\,\tilde V^{(1)}_{aa'}}
=
d_1\sqrt{2K_2\,\tilde V^{(1)}_{aa'}}.
\]
Dividing by $d_2$ yields
\[
|u^\top v|
\le
\frac{d_1}{d_2}\sqrt{2K_2\,\tilde V^{(1)}_{aa'}}.
\]

The second bound,
$|u^\top v|\le \frac{d_2}{d_1}\sqrt{2K_1\,\tilde V^{(2)}_{bb'}}$,
follows by the same argument after swapping the roles of the two tasks (and swapping the pairs $(a,a')$ and $(b,b')$). Taking the minimum of the two bounds proves the claim.
\end{proof}

\section{Optimality of the leading constant $4$}
\label{app:opt_const_4}

This section shows that the leading coefficient $4$ multiplying $\tilde V_{ij}$ in
Theorems~\ref{thm:eq-wts} and~\ref{thm:ncc-full-optimized} cannot be improved under only second-moment information.
The key point is that, in the known-centroid limit, pairwise NCC error is a one-sided tail probability of a
mean-zero scalar random variable with known variance. The sharp distribution-free bound is given by Cantelli's
(one-sided Chebyshev) inequality, and it is tight via a two-point construction. Therefore any bound that depends
only on $\EE[X]=0$ and $\Var(X)$ must incur the factor $4$ in the small-$\tilde V_{ij}$ regime, unless additional
tail assumptions are imposed.

\subsection{Reduction to a one-dimensional tail event}

Consider the idealized regime where the class centroids are known, equivalently $m\to\infty$ so that
$\widehat\mu_i=\mu_i$ and $\widehat\mu_j=\mu_j$.
For a test point $z_i\sim D_i$, define the axis projection
\[
X := (z_i-\mu_i)^\top u_{ij},
\qquad
\Var(X)=u_{ij}^\top\Sigma_i u_{ij}=d_{ij}^2\,\tilde V_{ij}.
\]
The pairwise NCC margin becomes
\[
\Delta_\infty
:= \|z_i-\mu_j\|_2^2-\|z_i-\mu_i\|_2^2
= d_{ij}^2 - 2d_{ij}X,
\]
so the pairwise error is exactly
\begin{equation}
\label{eq:pairwise_tail_event_app}
p^{\mathrm{NCC}}_{i\to j}
=\Pr(\Delta_\infty\le 0)
=\Pr(X\ge d_{ij}/2).
\end{equation}
Thus, bounding $p^{\mathrm{NCC}}_{i\to j}$ reduces to bounding a one-sided tail probability for a
mean-zero scalar random variable with a known variance.

\subsection{Cantelli's inequality and the sharp bound}

\begin{lemma}[Cantelli / one-sided Chebyshev]
\label{lem:cantelli_app}
Let $X$ be a real random variable with $\EE[X]=0$ and $\Var(X)=\sigma^2<\infty$.
Then for any $t>0$,
\[
\Pr(X\ge t)\ \le\ \frac{\sigma^2}{\sigma^2+t^2}.
\]
Moreover, this bound is tight: for any $\sigma^2>0$ and $t>0$, there exists a two-point distribution on
$\{-a,t\}$ (with a suitable $a>0$) that attains equality.
\end{lemma}

Applying Lem.~\ref{lem:cantelli_app} to \eqref{eq:pairwise_tail_event_app} with $t=d_{ij}/2$ and
$\sigma^2=d_{ij}^2\tilde V_{ij}$ yields
\begin{equation}
\label{eq:cantelli_bound_tildeV_app}
p^{\mathrm{NCC}}_{i\to j}
~\le~
\frac{d_{ij}^2\tilde V_{ij}}{d_{ij}^2\tilde V_{ij}+d_{ij}^2/4}
=
\frac{4\tilde V_{ij}}{1+4\tilde V_{ij}}
~\le~
4\tilde V_{ij}.
\end{equation}
The last inequality is the linearization useful for small $\tilde V_{ij}$.
Importantly, the fraction $\frac{4\tilde V_{ij}}{1+4\tilde V_{ij}}$ is the best possible distribution-free
upper bound given only $\EE[X]=0$ and $\Var(X)=d_{ij}^2\tilde V_{ij}$.

\subsection{Minimax tightness and the necessity of the factor $4$}

\begin{proposition}[Sharpness of the coefficient $4$ under second moments]
\label{prop:sharp_4_app}
Fix $d_{ij}>0$ and $\tilde V_{ij}>0$.
Among all real random variables $X$ satisfying $\EE[X]=0$ and $\Var(X)=d_{ij}^2\tilde V_{ij}$,
the maximal value of $\Pr(X\ge d_{ij}/2)$ equals $\frac{4\tilde V_{ij}}{1+4\tilde V_{ij}}$.
\end{proposition}

\begin{proof}
The upper bound is exactly Cantelli's inequality in Lem.~\ref{lem:cantelli_app} applied with
$t=d_{ij}/2$ and $\sigma^2=d_{ij}^2\tilde V_{ij}$, giving
$\Pr(X\ge d_{ij}/2)\le \frac{4\tilde V_{ij}}{1+4\tilde V_{ij}}$.
Tightness follows from the extremal two-point construction in Lem.~\ref{lem:cantelli_app}:
let $X$ take values $t=d_{ij}/2$ and $-a$ with probabilities $p$ and $1-p$, where
$a=\sigma^2/t$ and $p=\sigma^2/(\sigma^2+t^2)$ with $\sigma^2=d_{ij}^2\tilde V_{ij}$.
This choice satisfies $\EE[X]=0$ and $\Var(X)=\sigma^2$, and yields
$\Pr(X\ge t)=p=\sigma^2/(\sigma^2+t^2)=\frac{4\tilde V_{ij}}{1+4\tilde V_{ij}}$.
\end{proof}

As $\tilde V_{ij}\downarrow 0$, the worst-case probability behaves as
$\frac{4\tilde V_{ij}}{1+4\tilde V_{ij}}=4\tilde V_{ij}+o(\tilde V_{ij})$,
so no uniform inequality of the form
$\Pr(X\ge d_{ij}/2)\le c\,\tilde V_{ij}$ can hold for all distributions with $c<4$.

\subsection{Connection to the finite-shot denominator $\Denij$}

In the finite-shot setting, $\EE[\Delta]=d_{ij}^2+\frac{v_j-v_i}{m}=d_{ij}^2(1+\alpha_{ij})$ where
$\alpha_{ij}:=\frac{v_j-v_i}{m d_{ij}^2}$, and our bounds normalize by $\Denij=(1+\alpha_{ij})^2$.
Heuristically, the effective threshold in the one-dimensional reduction is rescaled from $d_{ij}/2$ to
$\frac{d_{ij}}{2}(1+\alpha_{ij})$, so the second-moment-optimal leading coefficient scales as
\[
\frac{4}{(1+\alpha_{ij})^2}
=
\frac{4}{\Denij},
\]
matching the leading term in Theorems~\ref{thm:eq-wts} and~\ref{thm:ncc-full-optimized}.

\end{document}